\documentclass{article}

 \usepackage[preprint]{neurips_2026}

\usepackage[utf8]{inputenc} 
\usepackage[T1]{fontenc}    
\usepackage{url}            
\usepackage{booktabs}       
\usepackage{amsfonts}       
\usepackage{nicefrac}       
\usepackage{microtype}      
\usepackage{xcolor}         
\usepackage{wrapfig}
\usepackage{tabularx}

\usepackage{tikz}
\usetikzlibrary{arrows.meta,positioning,calc}

\definecolor{methodred}{HTML}{FF0902}
\definecolor{methodblue}{HTML}{0000FF}
\definecolor{methodgreen}{HTML}{01C226}

\definecolor{PWOcolor}{HTML}{2F2C9B}

\usepackage{amsmath}
\usepackage{amsthm}
\usepackage{longtable}
\usepackage[parfill]{parskip}
\usepackage{physics}
\usepackage{mathrsfs}
\usepackage{bbm}
\usepackage{thmtools}

\usepackage{graphicx}
\usepackage{subcaption}
\usepackage{adjustbox}

\usepackage[ruled,vlined]{algorithm2e}
\SetAlgoHangIndent{0pt}
\usepackage[dvipsnames]{xcolor}

\newcommand{\I}{\mathbbm{1}}

\newcommand{\R}{\mathbb{R}}

\newcommand{\E}{\mathbb{E}}

\DeclareMathOperator*{\argmin}{argmin}
\DeclareMathOperator{\atanTwo}{\atan_2}

\newtheorem{theorem}{Theorem}[section]

\newtheorem{lemma}[theorem]{Lemma}

\usepackage[
    colorlinks=true,
    linkcolor=blue,
    citecolor=blue,
    urlcolor=blue
]{hyperref}

\usepackage[toc,page,header]{appendix}
\usepackage[nohints]{minitoc}

\usepackage[normalem]{ulem} 

\title{One More Time: Revisiting Neural Quantum States from a Reinforcement Learning Perspective}

\author{%
Juan Agustín Duque$^{1,2,\ast}$ \quad
Sergio García-Heredia$^{2,\ast}$ \quad
Vinicius Hernandes$^{3}$ \quad
Eliška Greplová$^{3}$ \\
\textbf{Thomas Spriggs}$^{3,\dagger}$ \quad
\textbf{Aaron Courville}$^{1, 4, \dagger}$ \quad
\textbf{Anna Dawid}$^{2,\dagger}$ \\
$^{1}$Mila Quebec AI Institute, Université de Montréal, Montréal, Canada \\
$^{2}$Applied Quantum Algorithms $\langle \mathrm{aQa}^L\rangle$, LIACS \& LION, Leiden University, Leiden, Netherlands \\
$^{3}$QuTech and Kavli Institute of Nanoscience, Delft University of Technology, Delft, Netherlands \\
$^{4}$CIFAR AI Chair\\
\texttt{juanduquevan@gmail.com} \\
{\small $^{\ast}$Equal contribution. $^{\dagger}$Equal supervision.}
}

\begin{document}

\maketitle

\begin{abstract}
Neural quantum states (NQS) provide a flexible and scalable framework for approximating quantum many-body wavefunctions. Among NQS parameterizations, autoregressive models are especially attractive because they enable exact, independent sampling from the Born distribution, avoiding the autocorrelation and mixing issues of Markov chain methods. Yet their optimization remains comparatively underexplored: Adam is a scalable method but ignores function space geometry, while stochastic reconfiguration is principled but costly and numerically fragile in large models. To address this gap, we show that variational energy minimization can be viewed as an advantage policy-gradient problem over the Born distribution, motivating trust-region optimization for NQS training. We introduce \emph{Proximal Wavefunction Optimization} (PWO), a principled trust-region algorithm that clips probability-ratio changes in the amplitude channel and phase increments in the phase channel. PWO avoids explicit matrix inversion, reuses samples across multiple updates, and combines the scalability of first-order optimization with theoretical guarantees. Across Ising and frustrated $J_1$--$J_2$ one- and two-dimensional spin systems, PWO improves stability and wall-clock convergence over Adam, minSR, and SPRING. Finally, we fine-tune a $1.5$B-parameter RWKV-7 model, demonstrating NQS optimization at a scale over three orders of magnitude beyond prior work.
\end{abstract}

\section{Introduction}

Quantum physics seeks to predict and understand the behavior of many interacting quantum particles. 
However, determining the exact ground (lowest-energy) state of a system of $N$ quantum bits generally requires diagonalizing a $2^N \times 2^N$ Hamiltonian matrix, quickly rendering exact methods intractable at scale. Variational approaches address this exponential complexity by optimizing a parameterized wavefunction ansatz to approximate the ground state. Neural quantum states (NQS) \citep{Carleo_2017, lange2024architecturesapplicationsreviewneural} build upon this by using neural networks as expressive wavefunction representations, capable of capturing complex correlations \citep{nomura2021dirac-type-d78} and high entanglement \citep{gauvin-ndiaye2025mott-e75}, as well as scaling to high-dimensional systems \citep{pescia2024message-passing-11d}, where traditional methods struggle.

The practical success of NQS depends not only on expressivity but also on the ability to sample from and optimize the variational distribution. State-of-the-art NQS methods \citep{wuVariationalBenchmarksQuantum2024} primarily rely on Markov Chain Monte Carlo (MCMC) to estimate energies and gradients, which can introduce autocorrelation, slow mixing, and unreliable exploration in difficult regimes \citep{wolff1990critical-079, debbio2004critical-247}. Autoregressive NQS remove this sampling bottleneck by factorizing the Born distribution and enabling exact, independent sampling \citep{sharir2020, hibat2020recurrent, moss2025leveraging-475}. As a result, they provide a clean setting for variational optimization: samples can be drawn directly from the current wavefunction, and training dynamics are no longer confounded by Markov chain limitations. Despite these advantages, their adoption remains limited, as stable optimization is a central bottleneck for training accurate autoregressive models.

Existing optimization methods expose a sharp trade-off. First-order optimizers such as Adam \citep{kingma2017adammethodstochasticoptimization} are computationally efficient and scale naturally to large neural networks, but they ignore the geometry of the variational wavefunction and can converge unstably or inaccurately in NQS applications \citep{pfau2020ab, liu2025efficient-31b}. Stochastic reconfiguration (SR), and scalable variants such as minSR \citep{Chen_2024}, are more geometrically principled because they approximate natural-gradient descent in wavefunction space. However, they require solving large and often ill-conditioned linear systems, with costs that become prohibitive for large networks or large sample regimes. Thus, autoregressive NQS resolve an important sampling problem, but stable first-order optimization of large autoregressive wavefunction models remains a central obstacle, which has been approached recently with transfer learning \citep{merali2026parallelscanrecurrentneural}.

\textbf{Contributions.}
This gap points to a missing optimization principle in NQS. In this work, we observe that the variational objective in NQS is mathematically equivalent to a policy-gradient objective in reinforcement learning (RL) under certain assumptions. While this connection has been implicitly present since early work \citep{Carleo_2017}, it has not been formalized or leveraged to design modern optimizers. 
In this work, we show that the
\emph{RL-style trust-region optimization improves NQS training scalability and stability.} Our contributions are as follows:
\begin{itemize}
    \item We formally connect variational energy minimization and policy-gradient RL, showing that the NQS gradient admits an advantage-weighted form over the Born distribution.
    \item We introduce \emph{Proximal Wavefunction Optimization (PWO)}, an algorithm for NQS training inspired by Proximal Policy Optimization (PPO) \citep{schulman2017proximal}, and prove that the PWO surrogate satisfies a trust-region improvement bound, allowing sample reuse.
    \item We show that PWO improves the stability and convergence speed of autoregressive NQS compared to Adam and minSR on standard benchmarks (1D, 2D Ising and $J_1$--$J_2$).
    \item Finally, we demonstrate the scalability of PWO by fine-tuning a 1.5 billion parameter RWKV-7 LLM \citep{peng2025rwkv7gooseexpressivedynamic} on the 1-D Ising model.
\end{itemize}

\section{Background}\label{s:background}

\subsection{Reinforcement Learning}

Reinforcement Learning (RL) is a machine learning paradigm in which an agent learns to make decisions by interacting with an environment. At each time step \( t \), the agent observes a state \( s_t \), selects an action \( a_t \), and then receives a reward \( r_{t+1} \) and a new state \( s_{t+1} \) from the transition function $P(\cdot| s_t, a_t)$. The goal of the agent is to learn a \emph{policy} that maximizes the expected return over time. Let \( \pi_{\boldsymbol{\theta}} \) denote a policy parameterized by \( \boldsymbol{\theta}\). Formally, the return is defined over a trajectory \( \tau := (s_0, a_0, r_1, s_1, a_1, r_2, \dots) \), which is a sequence of states, actions, and rewards generated by the agent-environment interaction. Following the notation of \cite{agarwal2021reinforcement}, the probability of a trajectory \( \tau \) under \( \pi_{\boldsymbol{\theta}} \) is given by
\begin{equation}
    \mathcal{P}_{\boldsymbol{\theta}}(\tau) = \mu(s_0) \pi_{\boldsymbol{\theta}}(a_0 | s_0) P(s_1 | s_0, a_0)\pi_{\boldsymbol{\theta}}(a_1 | s_1) P(s_2 | s_1, a_1)\hdots,
\end{equation}

where $\mu(s_0)$ denotes the initial distribution over states. The objective of RL is to find a policy that maximizes the expected \emph{discounted return}, which is captured by the state-value and action-value functions of a policy $\pi_{\boldsymbol{\theta}}$:
\begin{equation}
    V^{\pi_{\boldsymbol{\theta}}}(s)
    :=
    \mathbb{E}_{\pi_{\boldsymbol{\theta}}}\!\left[\sum_{t=0}^{\infty}\gamma^t r_{t}\,\middle|\, s_0=s\right],
    \quad
    Q^{\pi_{\boldsymbol{\theta}}}(s,a)
    :=
    \mathbb{E}_{\pi_{\boldsymbol{\theta}}}\!\left[\sum_{t=0}^{\infty}\gamma^t r_{t}\,\middle|\, s_0=s,\ a_0=a\right],
\end{equation}
respectively. In policy optimization, an agent maximizes its expected return by performing gradient ascent with a \emph{policy gradient}  estimator \citep{williams1992simple} of the form:
\begin{equation}
    \label{eq:REINFORCE}
    \grad_{\boldsymbol{\theta}}V^{\pi_{\boldsymbol{\theta}}}(\mu) = \mathbb{E}_{\tau \sim  \mathcal{P}_{\boldsymbol{\theta}}} \left[\sum_{t=0}^\infty \gamma^t A^{\pi_{\boldsymbol{\theta}}}(s_t, a_t) \grad_{\boldsymbol{\theta}} \log \pi_{\boldsymbol{\theta}} (a_t | s_t)\right].
\end{equation}
Here $A^{\pi_{\boldsymbol{\theta}}}(s, a) := Q^{\pi_{\boldsymbol{\theta}}}(s, a) - V^{\pi_{\boldsymbol{\theta}}}(s)$ denotes the advantage of taking action $a$ while in state $s$, \textit{i.e.}, the expected return of taking action $a$ relative to the policy average while in state $s$.

\clearpage
\subsection{Neural Quantum States}

One of the key tasks in theoretical quantum many-body physics is to approximate the ground-state wavefunction of a system with many interacting particles. 
For a system of $N$ spin-$1/2$ particles, 
a configuration is a binary vector $\vb{s}\in\{\pm 1\}^N$, and the wavefunction can be viewed as a complex vector $\ket{\psi}$ in a Hilbert space $\mathcal{H}$ (see Appendix \ref{app:dirac})
indexed by all $2^N$ configurations. Equivalently, it defines a function
$\vb{s}\mapsto \psi(\vb{s})$, where $\psi(\vb{s})\in\mathbb{C}$ is the
amplitude assigned to configuration $\vb{s}$:
\begin{equation}
    \ket{\psi} = \sum_{\vb{s}\in\{\pm1\}^N} \psi(\vb{s}) \ket{\vb{s}},
\end{equation}
where we use the Dirac, bra-ket, notation (see Appendix \ref{app:dirac}). The exponential size of this vector makes explicit representations intractable for large $N$. Neural quantum states (NQS) address this by parameterizing the amplitude function with a neural network, $f_{\boldsymbol{\theta}}$, so that $f_{\boldsymbol{\theta}}(\vb{s})$ provides a compact approximation to $\psi(\vb{s})$ \citep{Carleo_2017,Dawid_2025}. In practice, NQS models output the logarithm of the complex amplitude,
\begin{equation}
   f_{\boldsymbol{\theta}}(\vb{s})
   = \log \psi_{\boldsymbol{\theta}}(\vb{s})
   = \log \abs{\psi_{\boldsymbol{\theta}}(\vb{s})}
     + i\,\arg\psi_{\boldsymbol{\theta}}(\vb{s}) ,
\end{equation}
which separates the wavefunction into a log-modulus and a phase, usually handled by separate networks (channels). This
representation is numerically convenient because amplitudes can vary over many
orders of magnitude. By the Born rule, normalized wavefunctions induce a
probability distribution over configurations,
$P_{\boldsymbol{\theta}}(\vb{s}) \propto
\abs{\psi_{\boldsymbol{\theta}}(\vb{s})}^2$. 
In this work, we focus on autoregressive NQS, which explicitly factorize $P_{\boldsymbol{\theta}}$ and enable exact independent sampling from the Born distribution.

\subsection{Finding the Ground State}
The Hamiltonian $\hat H$ is the energy operator of a quantum system: it encodes the interactions, external fields, and kinetic terms that determine which wavefunctions have low or high energy. A popular application of the NQS framework is to find the ground state of quantum systems. Given a description of a physics system within a Hamiltonian, $\hat{H}$, the ground state is the eigenstate $\ket{\psi_0}$ corresponding to the lowest eigenvalue $E_0$. 
This can be cast as a minimization problem:
\begin{equation}
    E_0\leq E[\psi] = \frac{\mel{\psi}{\hat{H}}{\psi}}{\braket{\psi}},\quad\forall\ket{\psi}\in\mathcal{H},
\end{equation}
where the equality holds if and only if $\ket{\psi}$ belongs to the ground-state eigenspace -- this is known as the \textit{variational principle}. Taking this variational formulation of the ground state as the starting point, so-called variational methods define a parametric family of states $\ket{\psi_{\boldsymbol{\theta}}}$ and approximate the ground state by solving 
\begin{equation}
    \label{eq:vmc_objective}
    \boldsymbol{\theta}^* = \argmin_{\boldsymbol{\theta}} E[\psi_{\boldsymbol{\theta}}] = \argmin_{\boldsymbol{\theta}} \frac{\langle \psi_{\boldsymbol{\theta}} | \hat{H} | \psi_{\boldsymbol{\theta}} \rangle}{\langle \psi_{\boldsymbol{\theta}} | \psi_{\boldsymbol{\theta}} \rangle}.
\end{equation}
In the case of NQS, this family of functions is given by a neural network, as already discussed. 

\subsection{Variational Monte Carlo}

To efficiently compute a minimizer of Eq.~\eqref{eq:vmc_objective}, NQS uses Monte Carlo estimates. This approach is known as variational Monte Carlo (VMC). In particular, the energy can be written as
\begin{equation}
    E[\psi_{\boldsymbol{\theta}}]
    = \frac{\mel{\psi_{\boldsymbol{\theta}}}{\hat{H}}{\psi_{\boldsymbol{\theta}}}}{\braket{\psi_{\boldsymbol{\theta}}}{\psi_{\boldsymbol{\theta}}}}
    = \sum_{\vb{s}} \mathcal{P}_{\boldsymbol{\theta}}(\vb{s})\, E^{\mathrm{loc}}_{\boldsymbol{\theta}}(\vb{s})
    = \E_{\vb{s}\sim \mathcal{P}_{\boldsymbol{\theta}}}\!\left[E^{\mathrm{loc}}_{\boldsymbol{\theta}}(\vb{s})\right],
    \label{eq:energy_expectation}
\end{equation}
where we define
\begin{equation}
    \mathcal{P}_{\boldsymbol{\theta}}(\vb{s})=\frac{\abs{\psi_{\boldsymbol{\theta}}(\vb{s})}^2}{\sum_{\vb{s}'}\abs{\psi_{\boldsymbol{\theta}}(\vb{s}')}^2}, \quad E^{\mathrm{loc}}_{\boldsymbol{\theta}}(\vb{s}) := \sum_{\vb{s}'}\frac{\psi_{\boldsymbol{\theta}}(\vb{s}')}{\psi_{\boldsymbol{\theta}}(\vb{s})}\mel{\vb{s}}{\hat{H}}{\vb{s}'}.
\end{equation}
With this formulation, NQS estimates the loss function by sampling a set of spin configurations $\{\vb{s}_i\}_{i=1}^M$ from the parameterized distribution $\mathcal{P}_{\boldsymbol{\theta}}$ and computing the sample mean of $E^{\mathrm{loc}}_{\boldsymbol{\theta}}(\mathbf{s}_i)$:
\begin{equation}
    L(\boldsymbol{\theta}) \approx \bar{E}_{\boldsymbol{\theta}}^{\mathrm{loc}} \equiv \frac{1}{M} \sum_{i=1}^M E^{\mathrm{loc}}_{\boldsymbol{\theta}}(\mathbf{s}_i).
\end{equation}
By differentiating Eq.~\eqref{eq:energy_expectation} (see Appendix \ref{app:variational_gradient}) it is possible to show that its gradient with respect to $\boldsymbol{\theta}$ 
can be expressed as an expected value over $\mathcal{P}_{\boldsymbol{\theta}}$ and estimated using either MCMC sampling or direct sampling of autoregressive models:
\begin{equation}
    \label{eq:vmc}
    \partial_{\theta_i}L(\boldsymbol{\theta}) = \E_{\vb{s}\sim\mathcal{P}_{\boldsymbol{\theta}}}\bqty{2\Re\Bqty{
    \pqty{E^{\mathrm{loc}}_{\boldsymbol{\theta}}(\mathbf{s}) - \mathbb{E}_{\mathbf{s}\sim\mathcal{P}_{\boldsymbol{\theta}}}[E^{\mathrm{loc}}_{\boldsymbol{\theta}}(\mathbf{s})]}O_i(\vb{s})^\ast
    }},
\end{equation}
for $i=1,\ldots, P$, where $O_i(\vb{s}) = \partial_{\theta_i}\log\psi_{\boldsymbol{\theta}}$ are the so-called score functions.


\subsection{Stochastic Reconfiguration}

Stochastic Gradient Descent (SGD) \citep{Robbins1951} is a simple and effective optimization method in many scenarios. However, in the context of VMC, it can suffer from slow convergence and instability due to the complex geometry of the parameter space. The key issue is that small changes in the parameters, $\boldsymbol{\theta}$, do not necessarily correspond to small changes in the quantum state $\ket{\psi_{\boldsymbol{\theta}}}$.

Stochastic reconfiguration (SR) \citep{sorella_green_1998}, which generalizes natural gradient descent \citep{amariNaturalGradientWorks1998} to VMC, preconditions the gradient $\grad_{\boldsymbol{\theta}}L(\boldsymbol{\theta})$ to ensure that the distance between $\ket{\psi_{\boldsymbol{\theta}}}$ at one training iteration and the next is small. This distance is given by the infidelity between the two states and approximated to second order by the Fubini-Study metric,
\begin{equation}
    S_{i j}= \operatorname{Re}\left\{\operatorname{Cov}_{\mathbf{s}\sim \mathcal{P}_{\boldsymbol{\theta}}}\left[O_{i}^{*}(\mathbf{s}), O_{j}(\mathbf{s})\right]\right\},\qquad i, j = 1,\ldots, P,
\end{equation}
which encodes the local geometry of the Hilbert space at the current parameter configuration. In mathematical terms, SR modifies the update rule in the following way:
\begin{equation}
    \boldsymbol{\theta}\leftarrow \boldsymbol{\theta}-\eta \mathbf{S}^{-1} \grad_{\boldsymbol{\theta}} L(\boldsymbol{\theta}),
\end{equation}
where $\eta\in\R$ is the learning rate. We discuss the high computational costs of SR and its more recent improvements in Appendix \ref{app:sr_analysis}, motivating the use of first-order optimization methods.

\section{Proximal Wavefunction Optimization}

\begin{algorithm}[t]
\DontPrintSemicolon
\caption{Proximal Wavefunction Optimization (PWO)}
\label{alg:PWO}
\KwIn{Hamiltonian $\hat H$, NQS $\psi_{\vb*{\theta}}$, batch size $M$, inner epochs $K$, amplitude clip $\epsilon$, phase clip $\delta$}
Initialize ${\boldsymbol{\theta}}$\;
\While{not converged}{
${\boldsymbol{\theta}}_{\mathrm{old}}\gets{\boldsymbol{\theta}}$\; sample $\{\vb{s}_i\}_{i=1}^M\sim\mathcal P_{{\boldsymbol{\theta}}_{\mathrm{old}}}$\;
Cache reference log-probabilities, $\log \mathcal P_{\boldsymbol{\theta}_{\mathrm{old}}}(\vb{s}_i)$; phases, $\arg\psi_{\boldsymbol{\theta}_{\mathrm{old}}}(\vb{s}_i)$; and normalized real and imaginary advantages, $A_{\boldsymbol{\theta}_{\mathrm{old}}}^{\mathrm{R}}(\vb{s}_i)$ and $A_{\boldsymbol{\theta}_{\mathrm{old}}}^{\mathrm{I}}(\vb{s}_i)$ respectively, of $\{\vb{s}_i\}_{i=1}^M$.\;
\For{$k=1,\dots,K$}{
Compute current log-probabilities $\log \mathcal P_{\boldsymbol{\theta}}(\vb{s}_i)$ and phases $\arg\psi_{\boldsymbol{\theta}}(\vb{s}_i)$ for $\{\vb{s}_i\}_{i=1}^M$\;
$r_i \gets \exp(\log \mathcal P_{\boldsymbol{\theta}}(\vb{s}_i)-\log \mathcal P_{\boldsymbol{\theta}_{\mathrm{old}}}(\vb{s}_i))$\;
$\phi_i \gets 2\cdot \operatorname{atan}_2\!\big(\sin(\arg\psi_{\boldsymbol{\theta}}(\vb{s}_i)-\arg\psi_{\boldsymbol{\theta}_{\mathrm{old}}}(\vb{s}_i)),\cos(\arg\psi_{\boldsymbol{\theta}}(\vb{s}_i)-\arg\psi_{\boldsymbol{\theta}_{\mathrm{old}}}(\vb{s}_i))\big)$\;
$\ell_i^R \gets \max\!\big(r_iA_{\boldsymbol{\theta}_{\mathrm{old}}}^{\mathrm{R}}(\vb{s}_i),\;\mathrm{clip}(r_i,1-\epsilon,1+\epsilon)A_{\boldsymbol{\theta}_{\mathrm{old}}}^{\mathrm{R}}(\vb{s}_i)\big)$\;
$\ell_i^I \gets \text{stop\_gradient}(r_i)\cdot\max\!\big(\phi_iA_{\boldsymbol{\theta}_{\mathrm{old}}}^{\mathrm{I}}(\vb{s}_i),\;\mathrm{clip}(\phi_i,-\delta,\delta)A_{\boldsymbol{\theta}_{\mathrm{old}}}^{\mathrm{I}}(\vb{s}_i)\big)$\;
$
{\boldsymbol{\theta}} \gets {\boldsymbol{\theta}} - \eta \grad_{\boldsymbol{\theta}} \frac1M\sum_{i=1}^M\big(\ell_i^R+\ell_i^I\big).
$
}
}
\end{algorithm}
Here, we introduce a new optimization procedure for training NQS. We first draw an equivalence between the current NQS optimization paradigm and the policy gradient update in RL, and then use this to motivate applying Proximal Policy Optimization to NQS.
\\
\begin{restatable}[Policy-gradient form of variational energy minimization]{proposition}{pgnqsprop}
\label{prop:pg_nqs}

Assume the Hamiltonian is stoquastic, \textit{i.e.} its matrix representation in a chosen computational basis has non-positive off-diagonal elements. Then we can assume $f_{\boldsymbol{\theta}} = \log\psi_{\boldsymbol{\theta}} = \log\abs{\psi_{\boldsymbol{\theta}}}$,
and the gradient of the variational energy can be written in policy-gradient form as
\begin{equation}
    \label{eq:pg_nqs}
    \grad_{\boldsymbol{\theta}} E[\psi_{\boldsymbol{\theta}}]
    = \E_{\vb{s}\sim\mathcal{P}_{\boldsymbol{\theta}}}\!\bqty{\pqty{E^{\mathrm{loc}}_{\boldsymbol{\theta}}(\mathbf{s}) - \mathbb{E}_{\mathbf{s}\sim\mathcal{P}_{\boldsymbol{\theta}}}[E^{\mathrm{loc}}_{\boldsymbol{\theta}}(\mathbf{s})]}\,\grad_{\boldsymbol{\theta}}\log \mathcal{P}_{\boldsymbol{\theta}}(\vb{s})}.
\end{equation}
In particular, variational energy minimization is equivalent to an advantage policy-gradient update over configurations. For a proof, see Appendix \ref{app:proof_of_pg_nqs}.
\end{restatable}

This proposition establishes a direct link between the NQS and RL frameworks (see also Tab.~\ref{tab:analogy}) by relating Eqs.~\eqref{eq:vmc} and~\eqref{eq:pg_nqs} which motivates our algorithm, \emph{Proximal Wavefunction Optimization} (PWO). The key observation is that the SR gradient is the result of solving the approximate constraint optimization problem of minimizing the expected energy while keeping the change in infidelity smaller than some quantity, $\delta$. An analogous problem has been studied in the RL literature for a long time, but referred to as \emph{trust region optimization} \citep{10.5555/645531.656005, schulman2017trust, schulman2017proximal}. The fundamental idea is similar to that of SR: maximizing the expected reward while keeping the change in the total variation distance smaller than some quantity, $\delta$.

\emph{Proximal Policy Optimization} (PPO) \citep{schulman2017proximal}, introduces a simple heuristic that empirically shows monotonic improvements by constraining updates to keep the updated policy close to the previous one. Applied to NQS, PPO would minimize the following surrogate loss:
\begin{equation}
    \label{eq:PPO}
    L_{\mathrm{mod}}^{\mathrm{clip}}(\boldsymbol{\theta}) = \mathbb{E}_{\vb{s} \sim  \mathcal{P}_{{\boldsymbol{\theta}}_{\mathrm{old}}}} \left[ \max\left(r_{\boldsymbol{\theta}}(\vb{s}) A_{\boldsymbol{\theta}_{\mathrm{old}}}^{\mathrm{R}}(\vb{s}), \text{clip}(r_{\boldsymbol{\theta}}(\vb{s}), 1-\epsilon, 1+\epsilon)A_{\boldsymbol{\theta}_{\mathrm{old}}}^{\mathrm{R}}(\vb{s}) \right) \right],
\end{equation}
where $r_{\boldsymbol{\theta}} = \mathcal P_{{\boldsymbol{\theta}}}/\mathcal P_{{\boldsymbol{\theta}}_{\mathrm{old}}}$ is the importance sampling ratio between the updated and previous distributions, $\epsilon$ is the amplitude clip, the max comes from solving a minimization problem, and
\begin{equation}
    A_{\boldsymbol{\theta}_{\mathrm{old}}}^{\mathrm{R}}(\vb{s}) := \Re\Bqty{E^{\mathrm{loc}}_{\boldsymbol{\theta}_{\mathrm{old}}}(\vb{s})} - \E_{\vb{s}\sim \mathcal{P}_{\boldsymbol{\theta}_{\mathrm{old}}}} \left[\Re\Bqty{E^{\mathrm{loc}}_{\boldsymbol{\theta}_{\mathrm{old}}}(\vb{s})}\right],
\end{equation}
is the equivalent of the advantage estimate, yielding the local energy of configuration $\vb{s}$ relative to the expected energy under the current NQS. In the case of RL, clipping ensures that policy updates remain conservative, thereby maintaining stability during optimization. 

While PPO is a natural fit for the amplitude channel, it does not by itself account for the full VMC gradient. As shown in Appendix~\ref{app:vmc_decomposition}, the variational gradient also contains an imaginary component, which governs how the phase of the wavefunction should evolve. We, therefore, introduce a second proximal objective that controls phase updates directly by constraining the wrapped phase increment between the current and reference models. Concretely, letting $\phi_{\boldsymbol{\theta}}(\vb{s})$ denote the wrapped phase difference (see Appendix~\ref{app:exact_complex_decomposition}), we minimize the clipped phase surrogate loss
\begin{equation}
    \label{eq:phase_clip}
    L_{\mathrm{arg}}^{\mathrm{clip}}(\boldsymbol{\theta})
    =
    \E_{\vb{s}\sim \mathcal{P}_{\boldsymbol{\theta}_{\mathrm{old}}}}
    \left[
    \operatorname{sg}\!\big(r_{\boldsymbol{\theta}}(\vb{s})\big)\,
    \max\!\left(
    \phi_{\boldsymbol{\theta}}(\vb{s})A_{\boldsymbol{\theta}_{\mathrm{old}}}^{\mathrm{I}}(\vb{s}),
    \mathrm{clip}\big(\phi_{\boldsymbol{\theta}}(\vb{s}),-\delta,\delta\big)A_{\boldsymbol{\theta}_{\mathrm{old}}}^{\mathrm{I}}(\vb{s})
    \right)
    \right],
\end{equation}
where $\operatorname{sg}(\cdot)$ denotes the stop-gradient operator (see Theorem~\ref{thm:pwo-first-order-consistency-main}), and
\begin{equation}
    A_{\boldsymbol{\theta}_{\mathrm{old}}}^{\mathrm{I}}(\vb{s}) := \Im\Bqty{E^{\mathrm{loc}}_{\boldsymbol{\theta}_{\mathrm{old}}}(\vb{s})} - \E_{\vb{s}\sim \mathcal{P}_{\boldsymbol{\theta}_{\mathrm{old}}}} \left[\Im\Bqty{E^{\mathrm{loc}}_{\boldsymbol{\theta}_{\mathrm{old}}}(\vb{s})}\right].
\end{equation}
The imaginary part of the local energy has zero expectation, so the centering is retained only to mirror the RL advantage. Equation~\eqref{eq:phase_clip} is the phase analog of PPO: instead of clipping a probability ratio, it clips the phase increment itself, preventing abrupt rotations of the wavefunction while still following the imaginary part of the VMC gradient. Using the phase increment instead of the ratio of phases ensures that the gradient of the combined loss (Eqs.~\eqref{eq:PPO}$+$\eqref{eq:phase_clip}) matches the original one (Eq.~\eqref{eq:vmc}) when we are on-policy, \textit{i.e.}, when the importance sampling ratio $r_{\boldsymbol{\theta}}$ is one (no sample reuse).

\begin{table}[t]
\vspace{-0.75em}
\centering
\normalsize
\setlength{\tabcolsep}{4pt}
\renewcommand{\arraystretch}{0.9}
\begin{tabularx}{\linewidth}{@{}
    >{\raggedright\arraybackslash}p{0.16\linewidth}
    >{\centering\arraybackslash}p{0.22\linewidth}
    >{\raggedright\arraybackslash}p{0.22\linewidth}
    >{\centering\arraybackslash}X
@{}}
\toprule
\multicolumn{2}{c}{\textbf{Reinforcement Learning}}
&
\multicolumn{2}{c}{\textbf{Variational Monte Carlo}} \\
\cmidrule(r){1-2}
\cmidrule(l){3-4}

Policy
&
$\pi_{\boldsymbol{\theta}}(a\mid s)$
&
Born distribution
&
$\mathcal{P}_{\boldsymbol{\theta}}(\vb{s})
\propto |\psi_{\boldsymbol{\theta}}(\vb{s})|^2$
\\

Advantage
&
$A^{\pi_{\boldsymbol{\theta}}}(s_t,a_t)$
&
Centered local energy
&
$\Delta E(\vb{s})
=
E^{\mathrm{loc}}_{\boldsymbol{\theta}}(\vb{s})
-
E[\psi_{\boldsymbol{\theta}}]$
\\

Policy gradient
&
$A^{\pi_{\boldsymbol{\theta}}}
\grad_{\boldsymbol{\theta}}
\log \pi_{\boldsymbol{\theta}}$
&
VMC force
&
$2\Re\!\left[
\Delta E(\vb{s})
\grad_{\boldsymbol{\theta}}
\log\psi_{\boldsymbol{\theta}}^\ast(\vb{s})
\right]$
\\

KL trust region
&
$D_{\mathrm{KL}}
(\pi_{\boldsymbol{\theta}_{\mathrm{old}}}
\|
\pi_{\boldsymbol{\theta}})$
&
Infidelity
&
$\mathcal{I}
(\psi_{\boldsymbol{\theta}},
\psi_{\boldsymbol{\theta}+\delta\boldsymbol{\theta}})$
\\

Fisher matrix
&
$\mathbf{F}$
&
Fubini--Study metric
&
$\mathbf{S}$
\\

\bottomrule
\end{tabularx}
\vspace{0.25em}
\caption{Correspondence between RL policy-gradient methods and VMC methods for NQS.}
\label{tab:analogy}
\vspace{-1em}
\end{table}

PWO combines the amplitude and phase losses into an \emph{efficient} (see Appendix \ref{app:pwo_speed}) first-order optimization procedure. For the amplitude channel, Proposition~\ref{prop:pg_nqs} shows that, when the phase is fixed, variational energy minimization takes the form of an advantage-weighted policy-gradient update over configurations, which naturally motivates a PPO-style clipped surrogate based on importance ratios. For the phase channel, Appendix~\ref{app:pwo_surrogate_bounds} shows that the imaginary part of the VMC gradient can be conservatively optimized through a surrogate that constrains the phase increment. In practice, PWO performs $K$ inner updates on a frozen batch of configurations sampled from the reference Born distribution $\mathcal P_{\boldsymbol{\theta}_{\mathrm{old}}}$, clipping probability-ratio changes for the amplitude and wrapped phase increments for the phase. Because the same batch is reused across inner steps, the phase objective includes detached importance weights to correct for the mismatch between the reference and current sampling distributions (see Theorem~\ref{thm:pwo-first-order-consistency-main}). Algorithm~\ref{alg:PWO} summarizes the resulting procedure.

\section{Theoretical Analysis}
\begin{figure}[t]
    \centering
    \begin{tikzpicture}[x=\linewidth,y=0.5625\linewidth]
        \node[anchor=south west,inner sep=0] at (0,0)
        {\includegraphics[width=\linewidth]{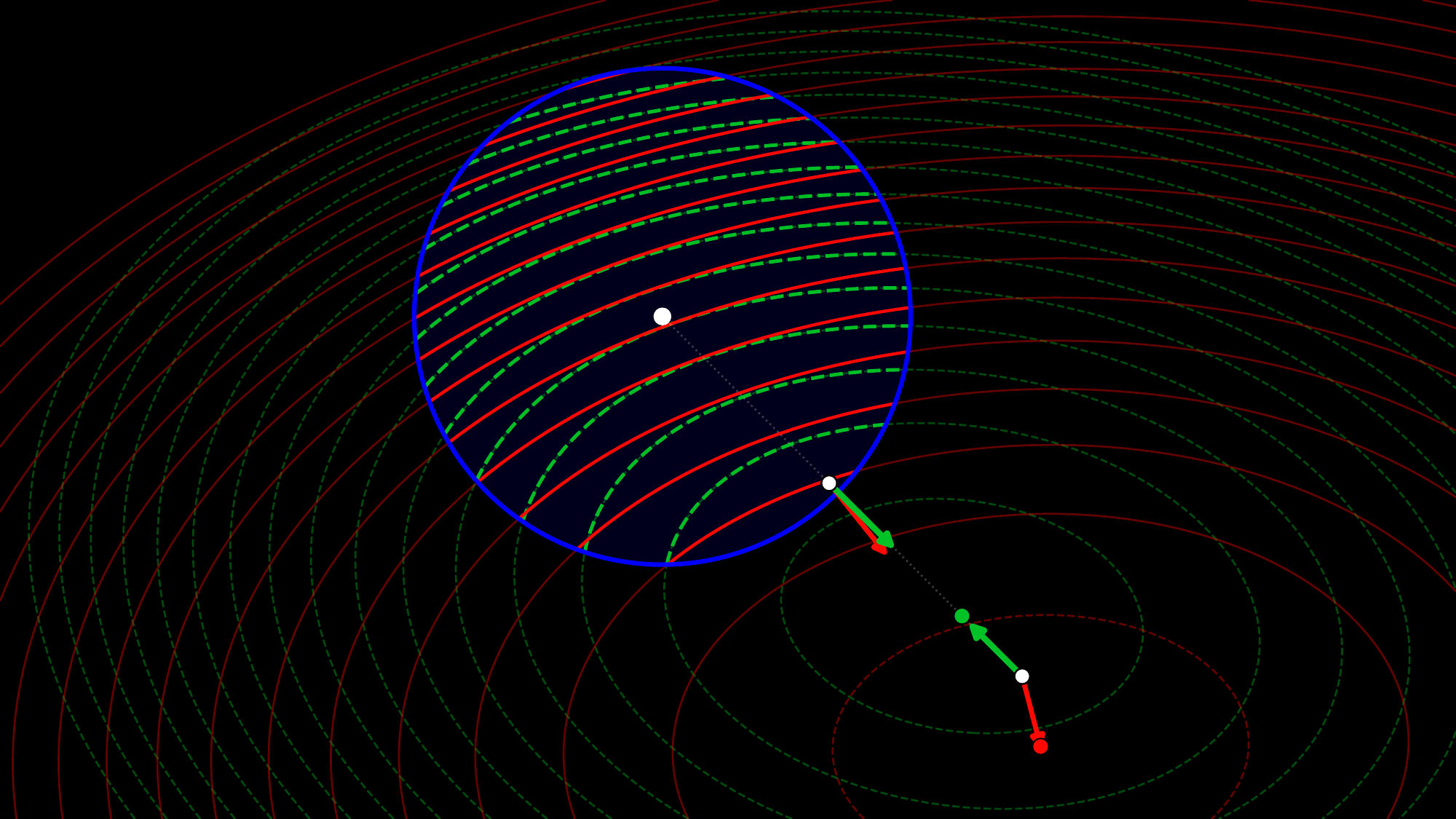}};


        \node[
            text=white,
            font=\sffamily\small\bfseries,
            align=center
        ] at (0.20,0.7)
        {old\\wavefunction\\ $\psi_{\boldsymbol{\theta}_{\rm old}}$};

        \node[
            text=white,
            font=\sffamily\small\bfseries,
            align=center
        ] at (0.5,0.16)
        {current\\wavefunction\\$\psi_{\theta}$};

        \draw[
            -{Latex[length=2mm]},
            white,
            thick
        ] (0.5,0.24) -- (0.565,0.405);

        \node[
            text=methodblue,
            font=\sffamily\normalsize\bfseries,
            align=center
        ] at (0.28,0.38)
        {trust\\region};

        \node[
            text=methodgreen,
            font=\sffamily\small\bfseries,
            align=center
        ] at (0.85,0.33)
        {true\\loss};

        \node[
            text=methodred,
            font=\sffamily\small\bfseries,
            align=center
        ] at (0.76,0.56)
        {surrogate\\loss};




        \draw[
            -{Latex[length=2mm]},
            white,
            thick
        ] (0.28,0.7) -- (0.45,0.62);





        \node[
            text=white,
            font=\small,
            align=center
        ] at (0.56,0.44)
        {inside $\nabla L(\theta) \approx \nabla L(\theta_\text{old})$};

        \node[
            text=white,
            font=\small,
            align=center
        ] at (0.84,0.19)
        {outside $\nabla L(\theta) \neq \nabla L(\theta_\text{old})$};

    \end{tikzpicture}

    \caption{
    Trust-region intuition for PWO. Inside the trust region, the surrogate landscape gives a reliable local improvement direction; outside the trust region, the surrogate and true
    landscapes may disagree substantially.
    }
    \label{fig:trust_region_cartoon}
\end{figure}
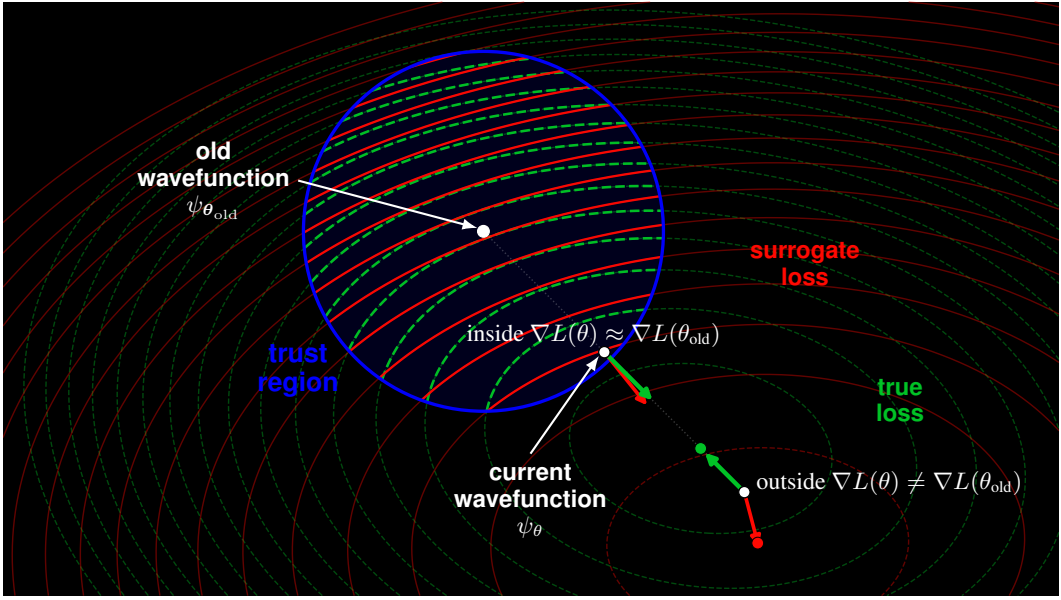

The theoretical justification for PWO parallels monotonic-improvement analyses in policy optimization. In CPI \citep{10.5555/645531.656005} and TRPO \citep{schulman2017trust}, the true return change is controlled by a surrogate objective under a reference policy plus a penalty for moving too far from it. PWO follows the same principle with the return replaced by negative variational energy, probability ratios controlling the amplitude channel, and phase increments controlling the phase channel. Figure~\ref{fig:trust_region_cartoon} illustrates this principle: trust regions make sample reuse possible by keeping optimization close enough to the reference wavefunction $\psi_{\boldsymbol{\theta}_{\rm old}}$ so that the surrogate loss function evaluated on old samples remains a controlled approximation to the current energy landscape. Theorem~\ref{thm:pwo-gradient-consistency-main} shows that this proximal construction preserves the exact local VMC direction. 
\begin{restatable}[First-order consistency of PWO]{theorem}{firstorderconsistency}
\label{thm:pwo-gradient-consistency-main}
Let $\ket{\psi_{\boldsymbol{\theta}}}$ be differentiable at
$\boldsymbol{\theta}_{\mathrm{old}}$, and assume common support. For $\epsilon,\delta>0$,
\begin{equation}
\left.
\grad_{\boldsymbol{\theta}}
\left(
L^{\mathrm{clip}}_{\mathrm{mod}}(\boldsymbol{\theta})
+
L^{\mathrm{clip}}_{\mathrm{arg}}(\boldsymbol{\theta})
\right)
\right|_{\boldsymbol{\theta}=\boldsymbol{\theta}_{\mathrm{old}}}
=
\left.
\grad_{\boldsymbol{\theta}}E[\psi_{\boldsymbol{\theta}}]
\right|_{\boldsymbol{\theta}=\boldsymbol{\theta}_{\mathrm{old}}}.
\end{equation}
For a proof see Appendix~\ref{app::pwo-first-order-consistency-main-app}.
\end{restatable}

Thus, the first PWO inner update exactly matches the VMC gradient; the surrogate only controls how far this direction is followed while reusing samples. For finite updates, local agreement is not enough. Let $r_{\boldsymbol{\theta}}=\mathcal P_{\boldsymbol{\theta}}/\mathcal P_{\boldsymbol{\theta}_{\mathrm{old}}}$ and $\alpha_{\boldsymbol{\theta}}$ be the wrapped phase increment and define
\begin{equation}
\label{eq:pwo_linear_term}
\mathcal A_{\boldsymbol{\theta}_{\mathrm{old}}}(\boldsymbol{\theta})
:=
2\E_{\vb{s}\sim\mathcal P_{\boldsymbol{\theta}_{\mathrm{old}}}}
\left[
\sqrt{r_{\boldsymbol{\theta}}}
\left(
\cos\alpha_{\boldsymbol{\theta}} A^{\mathrm R}_{\boldsymbol{\theta}_{\mathrm{old}}}
+
\sin\alpha_{\boldsymbol{\theta}} A^{\mathrm I}_{\boldsymbol{\theta}_{\mathrm{old}}}
\right)
\right],
\end{equation}

\begin{restatable}[Infidelity energy bound]{theorem}{infidelityimprovement}
\label{thm:pwo-infidelity-improvement}
Let $\ket{\psi_{\boldsymbol{\theta}_{\mathrm{old}}}}$ and
$\ket{\psi_{\boldsymbol{\theta}}}$ be normalized wavefunctions with common support. Then
\begin{equation}
\label{eq:pwo_infidelity_bound}
E[\psi_{\boldsymbol{\theta}}]-E[\psi_{\boldsymbol{\theta}_{\mathrm{old}}}]
\leq
\mathcal A_{\boldsymbol{\theta}_{\mathrm{old}}}(\boldsymbol{\theta})
+
2\norm{\hat H-E[\psi_{\boldsymbol{\theta}_{\mathrm{old}}}]\I}_{\infty}
\left(
1-\sqrt{1-\mathcal I(
\psi_{\boldsymbol{\theta}_{\mathrm{old}}},
\psi_{\boldsymbol{\theta}}
)}
\right).
\end{equation}
For a proof see Appendix \ref{app:exact_complex_decomposition}.
\end{restatable}

This bound is independent of clipping and holds for any finite update. The clipped PWO certificate follows by adding the stronger requirement that the realized update satisfies the amplitude and phase trust regions globally, not only on the sampled batch (which is encouraged by the method but not guaranteed, see Appendix \ref{app:pwo_clipped_improvement_certificate}).\\

\begin{restatable}[Clipped PWO improvement certificate]{corollary}{pwoimprovement}
\label{cor:pwo-clipped-improvement}
Assume the conditions of Theorem~\ref{thm:pwo-infidelity-improvement}, bounded centered local energies, and global constraints
$r_{\boldsymbol{\theta}}(\vb{s})\in[1-\epsilon,1+\epsilon]$ and
$|\alpha_{\boldsymbol{\theta}}(\vb{s})|\leq\delta$ for all $\vb{s}$, with
$0\leq\epsilon\leq1$ and $0\leq\delta\leq\pi$. Then there exists
$C_{\boldsymbol{\theta}_{\mathrm{old}}}<\infty$ such that
\begin{equation}
\label{eq:pwo_main_bound}
E[\psi_{\boldsymbol{\theta}}]-E[\psi_{\boldsymbol{\theta}_{\mathrm{old}}}]
\leq
L^{\mathrm{clip}}_{\mathrm{mod}}(\boldsymbol{\theta})
+
L^{\mathrm{clip}}_{\mathrm{arg}}(\boldsymbol{\theta})
+
2C_{\boldsymbol{\theta}_{\mathrm{old}}}
\left(
1-
\frac{1+\sqrt{1-\epsilon^2}}{2}
\cos^2\!\left(\frac{\delta}{2}\right)
\right).
\end{equation}
For a proof see Appendix~\ref{app:pwo_clipped_improvement_certificate}.
\end{restatable}
The final term is the trust-region penalty. It vanishes as $\epsilon,\delta\rightarrow0$ and grows as either clipping range is relaxed. Hence, if the clipped surrogate is minimized until the right-hand side of Eq.~\eqref{eq:pwo_main_bound} is negative, the updated wavefunction is guaranteed to have lower energy. Together, the results show that PWO follows the exact local VMC direction, while finite-update errors are controlled generally by infidelity and, under global clipping, by explicit amplitude and phase trust-regions. 
\section{Experiments}

\begin{figure}[t]
    \centering
    \begin{subfigure}[t]{0.49\linewidth}
        \centering
        \includegraphics[width=\linewidth]{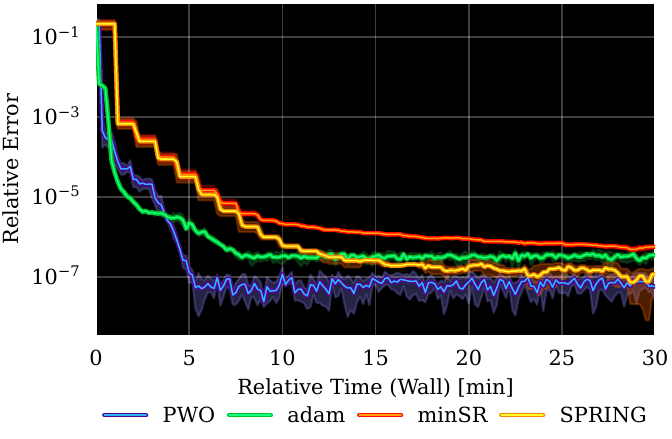}
        \label{fig:ising_rel_error}
    \end{subfigure}
    \hfill
    \begin{subfigure}[t]{0.49\linewidth}
        \centering
        \includegraphics[width=\linewidth]{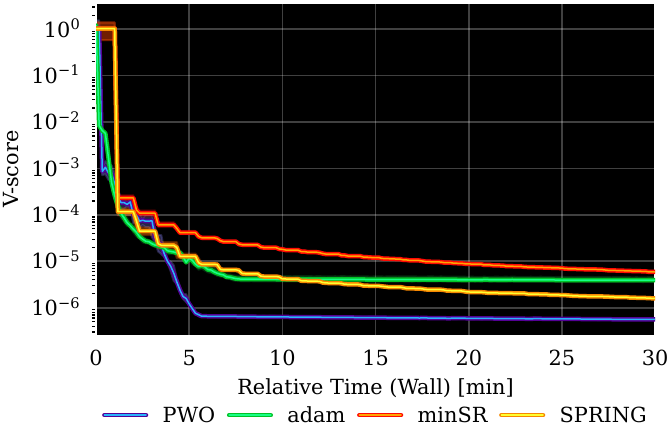}
        \label{fig:ising_v_score}
    \end{subfigure}
    \caption{Comparison of PWO, Adam, minSR, and SPRING on the transverse-field Ising model over $10$ random seeds. All methods were run with $1024$ samples on a single NVIDIA L40S GPU.}
    \label{fig:ising_comparison}
\end{figure}

\subsection{Spin-chain benchmarks}
We evaluate PWO for 1.5M-parameter autoregressive NQS on three different spin systems. Each systems consists of a one-dimensional spin-$1/2$ chain of $N=12$ sites with periodic boundary conditions, so that site indices are understood modulo $N$. Let
$\hat{\sigma}_i^\alpha$, $\alpha\in\{x,y,z\}$, denote the Pauli operator on site $i$,
and let $\hat{\vb S}_i=1/2(\hat{\sigma}_i^x,\hat{\sigma}_i^y,\hat{\sigma}_i^z)$.
For each Hamiltonian, we minimize the variational energy $E[\psi_{\boldsymbol{\theta}}]$ and report the relative error with respect to the exact ground-state energy $E_0$, $\epsilon_{\rm rel} = (E[\psi_{\boldsymbol{\theta}}] - E_0) / E_0$, which can be computed for these relatively small systems. We also report the V-score \citep{wuVariationalBenchmarksQuantum2024}, a scale-invariant convergence metric based on the energy variance that vanishes for exact eigenstates. The two Hamiltonians chosen offer increasing difficulty, as the underlying physics requires more complicated functions to represent. We use the same NQS architecture and the same number of Monte
Carlo samples for all optimizers, so differences in performance reflect optimization rather than
model capacity. For details about the hyperparameter search and the used architecture, see Appendices \ref{app:hyperparameter_search} and \ref{app:architecture}.

\textbf{Transverse Field Ising Model.}
\begin{equation}
    \hat H_{\mathrm{Ising}}
    =
    -J \sum_{i=1}^{N} \hat{\sigma}_i^z \hat{\sigma}_{i+1}^z
    -h \sum_{i=1}^{N} \hat{\sigma}_i^x.
    \label{eq:ising_hamiltonian}
\end{equation}
Here, $J=1$ is the ferromagnetic coupling strength and $h=1$ is the transverse-field strength. This Hamiltonian is a standard sign-problem-free benchmark: the diagonal interaction term favors
aligned spin configurations, while the transverse field introduces quantum fluctuations by flipping
individual spins. Therefore, Proposition \ref{prop:pg_nqs} holds exactly.


\textbf{Heisenberg $J_1$--$J_2$ Chain.}
\begin{equation}
    \hat H_{J_1\text{--}J_2}
    =
    J_1 \sum_{i=1}^{N}
    \hat{\vb S}_i \cdot \hat{\vb S}_{i+1}
    +
    J_2 \sum_{i=1}^{N}
    \hat{\vb S}_i \cdot \hat{\vb S}_{i+2},
    \qquad J_1>0,\; J_2>0 .
    \label{eq:j1j2_hamiltonian}
\end{equation}
We set $J_1=1$ and $J_2=0.5$, \textit{i.e.}, $J_2/J_1=0.5$. At this ratio the next-nearest-neighbor interaction maximally frustrates the nearest-neighbor antiferromagnetic coupling, placing the system at the Majumdar–Ghosh point, where a highly entangled ground state with a complex sign structure makes both variational optimization and Monte Carlo sampling particularly challenging \citep{sorella_green_1998}. Unlike $\hat{H}_{\mathrm{Ising}}$, the Hamiltonian $\hat H_{J_1\text{--}J_2}$ requires a nontrivial sign structure in the computational basis; we use a complex-valued parameterization so that PWO can learn this structure through its phase channel. This setting evaluates whether the clipped phase surrogate (Eq.~\eqref{eq:phase_clip}) remains a useful practical heuristic. We make an analogous study for $J_1 = 0.25$ and $J_2=0$ (the so-called Heisenberg chain) in Appendix \ref{app:individual_seeds_hams}.

All results are computed over $10$ random seeds. Since several runs reach high-precision errors, numerical instabilities can produce extreme outliers. Mean-and-standard-deviation summaries are therefore poorly suited to this regime, as a single unstable run can distort both the central estimate and the uncertainty band. Following \citet{agarwal2022deepreinforcementlearningedge}, we instead report interquartile statistics: at each evaluation time, curves show the interquartile mean (IQM), computed by averaging the middle $50\%$ of seeds after discarding the lowest and highest quartiles. Shaded regions denote the interquartile range, from the $25$th to the $75$th percentile. This provides robust central estimates while still showing run-to-run variability. The figures with all individual seeds plotted are shown in Appendices~\ref{app:individual_seeds_hams}-\ref{app:individual_seeds_LLM}. We discuss limitations of our approach in Appendix~\ref{app:limitations}.


\subsection{Convergence and Stability Results}

Figure \ref{fig:ising_comparison} shows that PWO converges the fastest of the four optimizers on the 1D Ising model, reaching relative error $10^{-7}$ in approximately $5$ minutes, while minSR requires approximately $30$ minutes to reach the same accuracy. The V-score follows the same trend: PWO rapidly suppresses energy fluctuations,
indicating that the proximal objective is not merely improving the energy estimate, but driving the
state toward an eigenstate. Adam and SPRING also make steady progress, but require more wall-clock
time to reach the same accuracy. minSR is stable on this easier Hamiltonian, but its per-step cost
makes it substantially slower in wall-clock time.


\begin{figure}[t]
    \centering
    \begin{subfigure}[t]{0.49\linewidth}
        \centering
        \includegraphics[width=\linewidth]{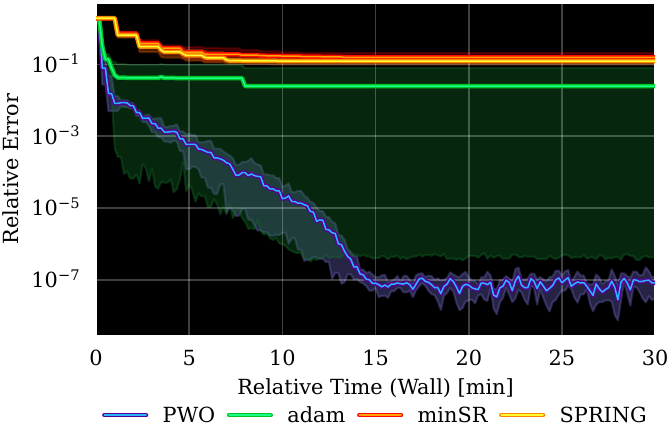}
        \label{fig:j1j2_rel_error}
    \end{subfigure}
    \hfill
    \begin{subfigure}[t]{0.49\linewidth}
        \centering
        \includegraphics[width=\linewidth]{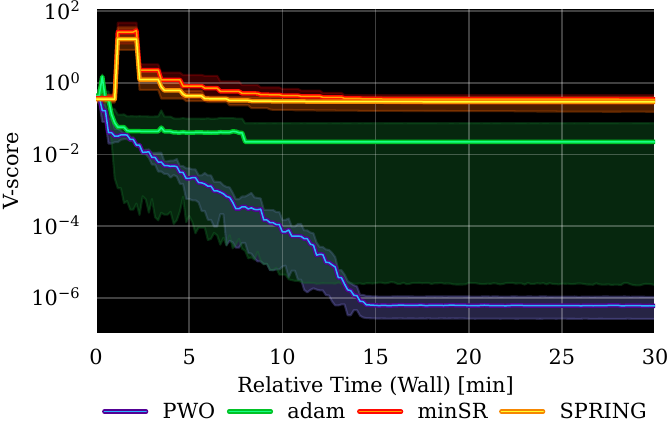}
        \label{fig:j1j2_v_score}
    \end{subfigure}
    \caption{Comparison of PWO, Adam, minSR, and SPRING on the Heisenberg $J_1$--$J_2$ chain over $10$ random seeds. MinSR is highly unstable on this Hamiltonian, with $6$ out of $10$ runs producing NaNs. All methods were run with $1024$ samples on a single NVIDIA L40S GPU.}
    \label{fig:j1j2_comparison}
\end{figure}

On the $J_1$--$J_2$ Hamiltonian in Figure \ref{fig:j1j2_comparison}, PWO reaches
relative error $10^{-7}$ in approximately $15$ minutes while maintaining a steadily decreasing V-score, showing that the learned state is converging in both energy and variance. In contrast, Adam plateaus several orders of
magnitude above PWO due to outliers. Also, the PWO contrast with minSR and SPRING is sharper than in the Ising case: minSR is numerically unstable, with $6$ out of $10$ runs producing NaNs, and both remain at a relative error of $10^{-1}$ after $30$ minutes. Despite using curvature-inspired updates, these methods struggle to make progress and exhibit numerical instability. 
This is the strongest empirical evidence for the central claim of the paper:
a PPO-style proximal objective can retain the low wall-clock cost of first-order optimization while
providing the stability needed to train expressive NQS on difficult quantum Hamiltonians. Experiments on the Heisenberg chain in Appendix \ref{app:individual_seeds_hams} offer the same conclusion. We also plot all individual seed curves without aggregation or filtering for greater clarity and visibility in Appendix \ref{app:individual_seeds_hams}.

\subsection{Two-dimensional $J_1$--$J_2$ model on the square lattice}
\label{sec:2d-j1j2}

\begin{figure}[t]
    \centering
    \includegraphics[width=0.485\linewidth]{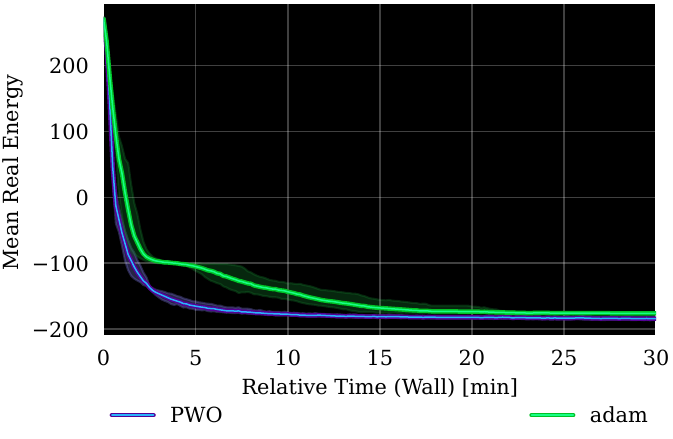}
    \hfill
    \includegraphics[width=0.485\linewidth]{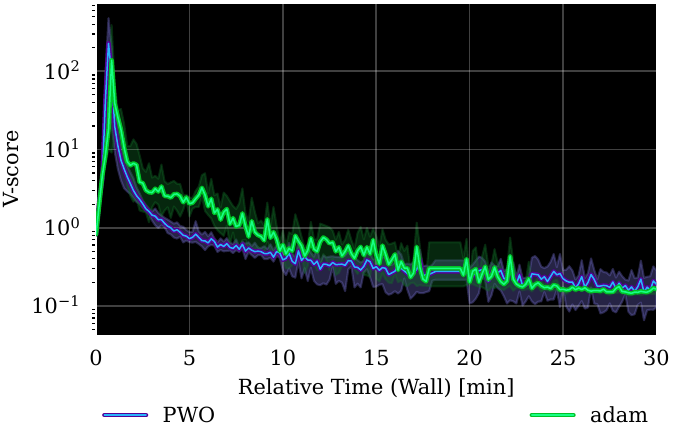}
    \caption{
    Two-dimensional frustrated $J_1$--$J_2$ Heisenberg model on the $10 \times 10$ lattice.
    Left: mean real energy. Right: V-score. PWO reaches lower energies faster than Adam
    and maintains a lower variance-based error signal over the same wall-clock budget.
    }
    \label{fig:j1j2-2d}
\end{figure}

We next test whether PWO remains effective beyond one-dimensional chains on a frustrated square-lattice $J_1$--$J_2$ Heisenberg model with periodic boundary conditions and system size of $10 \times 10$. We use a complex patch-autoregressive transformer with two-dimensional RoPE and enforce the zero-magnetization sector. Since exact diagonalization is intractable at this size, we compare optimizers using the mean real energy and the V-score. As shown in Fig.~\ref{fig:j1j2-2d}, PWO decreases the energy faster than Adam and reaches a lower-energy regime within the same wall-clock budget. The V-score follows the same trend, indicating that the proximal objective also improves stability on this frustrated two-dimensional benchmark. 
Individual seed runs are shown in Appendix \ref{app:individual_seeds_hams}.

\subsection{Scalability Experiments}

\begin{figure}[h]
    \centering
    \includegraphics[width=\textwidth]{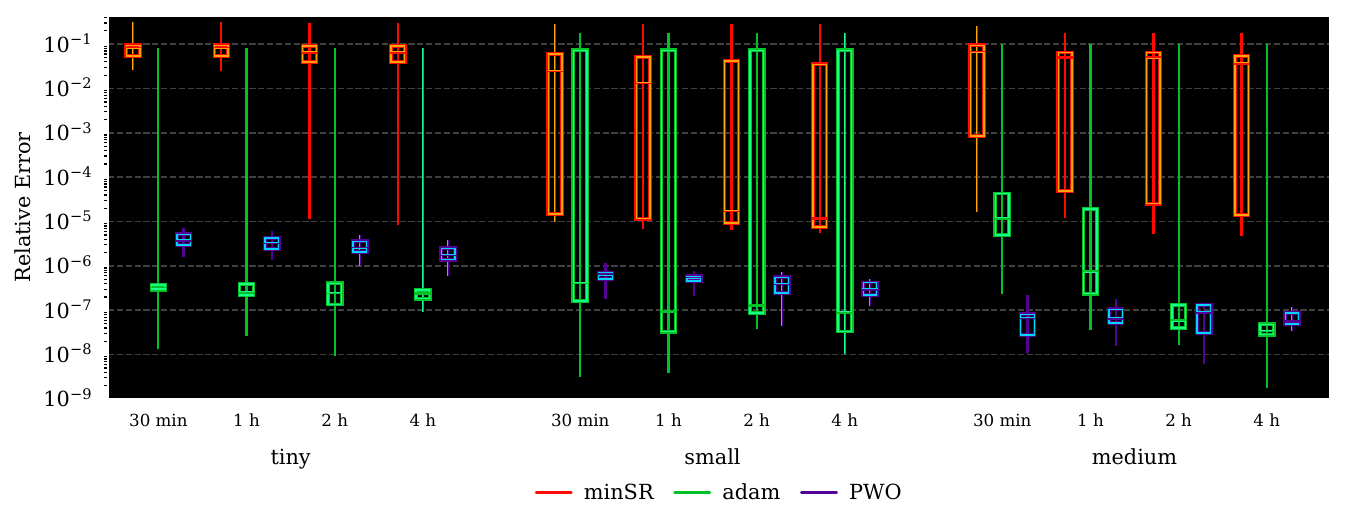}
    \caption{
        Wall-clock scaling comparison across model sizes and optimization methods.
        Boxplots show the interquartile mean relative error over seeds, with boxes indicating the interquartile range and lines indicating the min and max.
        Results are grouped by model size and wall-clock time, and run on a single NVIDIA A100 GPU.
    }
    \label{fig:scaling_plot}
\end{figure}

Figure~\ref{fig:scaling_plot} compares the scaling behavior of PWO, Adam, and minSR for increasing NQS model sizes on the Heisenberg $J_1$--$J_2$ chain (see Appendix~\ref{app:scaling_experiment} for details). PWO exhibits consistent stability and monotonic energy improvement as NQS size grows, reaching near-final energies within the first 30 minutes. In contrast, minSR is both unstable and never reaches competitive energy values. Adam performs best for the tiny network, but degrades in stability and efficiency for small and medium NQS, where PWO is both faster and more reliable. In particular, achieving comparable energy scales with the medium NQS requires roughly $4 \times$ more time with Adam than with PWO. In Appendix~\ref{app:pwo_speed}, we show that PWO converges faster than other optimizers because it does more optimization steps per unit of time. Although Adam attains lower energies after extended training of 4 hours, the observed scaling trend suggests that this advantage may diminish for yet larger networks. In Appendix \ref{app:additional_scaling} we perform additional scaling experiments on the number of samples and system size.

\subsection{NQS Fine-tuning of Large Language Models}

\begin{figure}[h]
    \centering
    \begin{subfigure}[t]{0.49\linewidth}
        \centering
        \includegraphics[width=\linewidth]{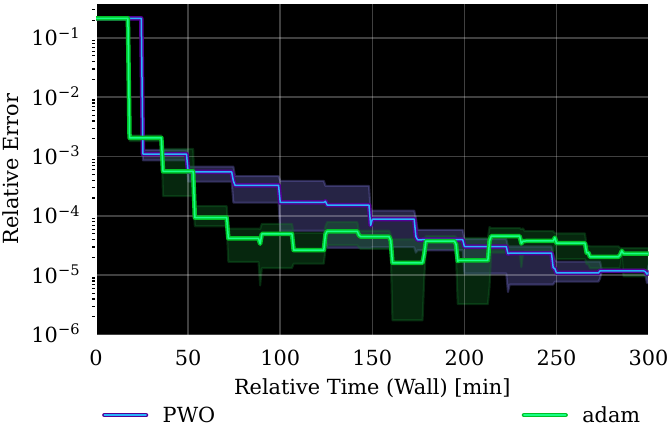}
        \label{fig:rwkv_rel_error}
    \end{subfigure}
    \hfill
    \begin{subfigure}[t]{0.49\linewidth}
        \centering
        \includegraphics[width=\linewidth]{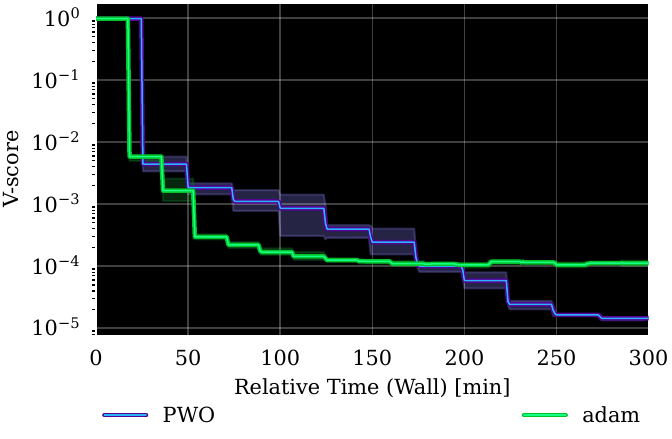}
        \label{fig:rwkv_v_score}
    \end{subfigure}
    \caption{Fine-tuning curves of a 1.5B-parameter RWKV7LLM on the 1-D Ising Model.
    }\label{fig:rwkv_comparison}
\end{figure}

Figure~\ref{fig:rwkv_comparison} shows fine-tuning results for a $1.5$B-parameter RWKV-7 model on the transverse-field Ising chain. PWO achieves a lower final relative error and V-score than Adam, while maintaining stable training across the full wall-clock budget. This improvement is not meant to establish that LLMs or large networks \citep{moss2026double-cec} are the right inductive bias for one-dimensional Hamiltonians; rather, it shows that the proximal objective is effective when the ansatz is scaled by more than three orders of magnitude beyond existing NQS in the literature \citep{rende2025foundation-7f9}.

\section{Related Work}

\textbf{Reinforcement Learning and Ground State Search}. Ground state search has been treated as an optimal control problem via Feynman-Kac representations \citep{pmlr-v107-barr20a}, and RL frameworks have been developed for lattice models by treating stoquastic Hamiltonians as reward functions \citep{pmlr-v145-gispen22a}. Recent meta-learning approaches \citep{jae2025reinforcement-2b9} leverage RL to optimize the quantum state learning process itself. However, these works apply RL by either reformulating the problem theoretically or optimizing the learning procedure. We instead use RL as the primary optimization framework applying it directly on the NQS ansatz.

\textbf{Importance Sampling and Trust Regions}. Sample reuse and proximal objectives have been previously explored for non-autoregressive NQS optimization. \cite{Yang_2020} introduced importance sampling gradient optimization (ISGO), which reuses Markov chain samples across multiple gradient updates by reweighting local energies with the probability ratio between the current and reference wavefunctions. Building on this, \cite{chen2022systematic} introduced a PPO-inspired a cosine penalty on phase variation. Our work formalizes the analogy with RL in Proposition \ref{prop:pg_nqs}, uses a different clipping objective directly on the imaginary gradient instead of a penalty, and makes the policy-gradient correspondence exact, allowing stable training of billion-parameter models for the first time.

\section{Conclusion}

We introduced Proximal Wavefunction Optimization (PWO), a scalable first-order optimizer for neural quantum states (NQS) derived from an explicit connection between variational Monte Carlo and reinforcement learning (RL). In the fixed-phase stoquastic setting, we showed that variational energy minimization can be written as an advantage policy-gradient objective over the Born distribution, where spin configurations play the role of actions and centered local energies act as advantages. This reformulates NQS optimization into a trust-region policy-optimization problem, motivating a PPO-style algorithm that clips probability-ratio changes in the amplitude channel and extends naturally to complex wavefunctions through a clipped phase-increment surrogate. Across spin-chain benchmarks, PWO improves wall-clock convergence and stability over Adam, minSR, and SPRING, with the largest gains in frustrated regimes where optimization is most fragile. Finally, by fine-tuning a $1.5$B-parameter RWKV-7 model as an autoregressive NQS, we bring NQS training into the regime of modern large-scale sequence modeling. 
More broadly, our results bridge the RL and NQS communities, enabling further mutual inspiration. 
By unlocking large autoregressive NQS, PWO opens the door to studying whether scale alone, in the spirit of the ``bitter lesson'', can overcome current barriers in quantum many-body simulation.

\section*{Reproducibility Statement}

Code for reproducing our experiments is available at \url{https://github.com/jduquevan/hyperscalenqs}. Full experimental details, including hyperparameter searches, model architectures, optimizer settings, sample budgets, and scaling configurations, are provided in Appendix~\ref{app:experimental_details}.

\section*{Acknowledgements}

This work was supported by the Netherlands Organization for Scientific Research (NWO/OCW), as part of Quantum Limits (project number SUMMIT.1.1016). This publication is also part of the project Optimal Digital-Analog Quantum Circuits with File No. NGF.1582.22.026 of the research programme NGF Quantum Delta NL 2022, which is (partly) financed by the Dutch Research Council (NWO) and the Dutch National Growth Fund initiative Quantum Delta NL. Juan Agustín Duque is supported by the St-Pierre-Larochelle Scholarship at the University of Montreal and by Aaron Courville’s CIFAR AI Chair in Representations that Generalize Systematically.

\bibliographystyle{plainnat}
\bibliography{refs_cited_only_uniform}

@book{Dawid_2025,
   title={Machine Learning in Quantum Sciences},
   ISBN={9781009504935},
   url={http://dx.doi.org/10.1017/9781009504942},
   DOI={10.1017/9781009504942},
   publisher={Cambridge University Press},
   author={Dawid, Anna and Arnold, Julian and Requena, Borja and Gresch, Alexander and Płodzień, Marcin and Donatella, Kaelan and Nicoli, Kim A. and Stornati, Paolo and Koch, Rouven and Büttner, Miriam and Okuła, Robert and Muñoz-Gil, Gorka and Vargas-Hernández, Rodrigo A. and Cervera-Lierta, Alba and Carrasquilla, Juan and Dunjko, Vedran and Gabrié, Marylou and Huembeli, Patrick and van Nieuwenburg, Evert and Vicentini, Filippo and Wang, Lei and Wetzel, Sebastian J. and Carleo, Giuseppe and Greplová, Eliška and Krems, Roman and Marquardt, Florian and Tomza, Michał and Lewenstein, Maciej and Dauphin, Alexandre},
   year={2025},
   month=jun 
}

@article{Robbins1951,
  author    = {Herbert Robbins and Sutton Monro},
  title     = {A Stochastic Approximation Method},
  journal   = {The Annals of Mathematical Statistics},
  year      = {1951},
  volume    = {22},
  number    = {3},
  pages     = {400--407},
  doi       = {10.1214/aoms/1177729586}
}

@article{williams1992simple,
    author = {Williams, Ronald},
    journal = {Machine Learning},
    pages = {229--256},
    title = {Simple statistical gradient-following algorithms for connectionist reinforcement learning},
    volume = {8},
    year = {1992}
}

@article{Carleo_2017,
   title={Solving the quantum many-body problem with artificial neural networks},
   volume={355},
   ISSN={1095-9203},
   url={http://dx.doi.org/10.1126/science.aag2302},
   DOI={10.1126/science.aag2302},
   number={6325},
   journal={Science},
   publisher={American Association for the Advancement of Science (AAAS)},
   author={Carleo, Giuseppe and Troyer, Matthias},
   year={2017},
   month=feb, pages={602–606}
}

@article{Chen_2024,
   title={Empowering deep neural quantum states through efficient optimization},
   volume={20},
   ISSN={1745-2481},
   url={http://dx.doi.org/10.1038/s41567-024-02566-1},
   DOI={10.1038/s41567-024-02566-1},
   number={9},
   journal={Nature Physics},
   publisher={Springer Science and Business Media LLC},
   author={Chen, Ao and Heyl, Markus},
   year={2024},
   month=jul, pages={1476–1481} }

@misc{schulman2017trust,
    title={Trust Region Policy Optimization}, 
    author={John Schulman and Sergey Levine and Philipp Moritz and Michael I. Jordan and Pieter Abbeel},
    year={2017},
    eprint={1502.05477},
    archivePrefix={arXiv},
    primaryClass={cs.LG}
}

@misc{schulman2017proximal,
    title={Proximal Policy Optimization Algorithms}, 
    author={John Schulman and Filip Wolski and Prafulla Dhariwal and Alec Radford and Oleg Klimov},
    year={2017},
    eprint={1707.06347},
    archivePrefix={arXiv},
    primaryClass={cs.LG}
}

@misc{kingma2017adammethodstochasticoptimization,
    title={Adam: A Method for Stochastic Optimization}, 
    author={Diederik P. Kingma and Jimmy Ba},
    year={2017},
    eprint={1412.6980},
    archivePrefix={arXiv},
    primaryClass={cs.LG},
    url={https://arxiv.org/abs/1412.6980}, 
}

@misc{lange2024architecturesapplicationsreviewneural,
      title={From Architectures to Applications: A Review of Neural Quantum States}, 
      author={Hannah Lange and Anka Van de Walle and Atiye Abedinnia and Annabelle Bohrdt},
      year={2024},
      eprint={2402.09402},
      archivePrefix={arXiv},
      primaryClass={cond-mat.dis-nn},
      url={https://arxiv.org/abs/2402.09402}, 
}

@book{agarwal2021reinforcement,
    title={Reinforcement Learning: Theory and Algorithms},
    author={Agarwal, Alekh and Jiang, Nan and Kakade, Sham and Sun, Wen},
    year={2021},
    publisher={Preprint}
}

@misc{smith2018superconvergencefasttrainingneural,
      title={Super-Convergence: Very Fast Training of Neural Networks Using Large Learning Rates}, 
      author={Leslie N. Smith and Nicholay Topin},
      year={2018},
      eprint={1708.07120},
      archivePrefix={arXiv},
      primaryClass={cs.LG},
      url={https://arxiv.org/abs/1708.07120}, 
}

@misc{agarwal2022deepreinforcementlearningedge,
      title={Deep Reinforcement Learning at the Edge of the Statistical Precipice}, 
      author={Rishabh Agarwal and Max Schwarzer and Pablo Samuel Castro and Aaron Courville and Marc G. Bellemare},
      year={2022},
      eprint={2108.13264},
      archivePrefix={arXiv},
      primaryClass={cs.LG},
      url={https://arxiv.org/abs/2108.13264}, 
}

@misc{merali2026parallelscanrecurrentneural,
      title={Parallel Scan Recurrent Neural Quantum States for Scalable Variational Monte Carlo}, 
      author={Ejaaz Merali and Mohamed Hibat-Allah and Mohammad Kohandel and Richard T. Scalettar and Ehsan Khatami},
      year={2026},
      eprint={2605.13807},
      archivePrefix={arXiv},
      primaryClass={cond-mat.str-el},
      url={https://arxiv.org/abs/2605.13807}, 
}

@inproceedings{10.5555/645531.656005,
    author = {Kakade, Sham and Langford, John},
    title = {Approximately Optimal Approximate Reinforcement Learning},
    year = {2002},
    isbn = {1558608737},
    publisher = {Morgan Kaufmann Publishers Inc.},
    address = {San Francisco, CA, USA},
    booktitle = {Proceedings of the Nineteenth International Conference on Machine Learning},
    pages = {267–274},
    numpages = {8},
    series = {ICML '02}
}

@article{Yang_2020,
   title={Deep learning-enhanced variational Monte Carlo method for quantum many-body physics},
   volume={2},
   ISSN={2643-1564},
   url={http://dx.doi.org/10.1103/PhysRevResearch.2.012039},
   DOI={10.1103/physrevresearch.2.012039},
   number={1},
   journal={Physical Review Research},
   publisher={American Physical Society (APS)},
   author={Yang, Li and Leng, Zhaoqi and Yu, Guangyuan and Patel, Ankit and Hu, Wen-Jun and Pu, Han},
   year={2020},
   month=Feb 
}

@article{rende2024simple-f44, 
  year    = {2024}, 
  title   = {A simple linear algebra identity to optimize large-scale neural network quantum states}, 
  author  = {Rende, Riccardo and Viteritti, Luciano Loris and Bardone, Lorenzo and Becca, Federico and Goldt, Sebastian}, 
  journal = {Communications Physics}, 
  doi     = {10.1038/s42005-024-01732-4}, 
  eprint  = {2310.05715}, 
  pages   = {260}, 
  number  = {1}, 
  volume  = {7}
}

@article{choo2020fermionic,
  title={Fermionic neural-network states for ab-initio electronic structure},
  author={Choo, Kenny and Mezzacapo, Antonio and Carleo, Giuseppe},
  journal={Nature communications},
  volume={11},
  number={1},
  pages={2368},
  year={2020},
  publisher={Nature Publishing Group UK London}
}

@article{hermann2020deep,
  title={Deep-neural-network solution of the electronic Schr{\"o}dinger equation},
  author={Hermann, Jan and Sch{\"a}tzle, Zeno and No{\'e}, Frank},
  journal={Nature Chemistry},
  volume={12},
  number={10},
  pages={891--897},
  year={2020},
  publisher={Nature Publishing Group UK London}
}

@article{pfau2020ab,
  title={Ab initio solution of the many-electron Schr{\"o}dinger equation with deep neural networks},
  author={Pfau, David and Spencer, James S and Matthews, Alexander GDG and Foulkes, W Matthew C},
  journal={Physical review research},
  volume={2},
  number={3},
  pages={033429},
  year={2020},
  publisher={APS}
}

@inproceedings{
glehn2023a,
title={A Self-Attention Ansatz for Ab-initio Quantum Chemistry},
author={Ingrid von Glehn and James S Spencer and David Pfau},
booktitle={The Eleventh International Conference on Learning Representations },
year={2023},
url={https://openreview.net/forum?id=xveTeHVlF7j}
}

@inproceedings{martens2015optimizing,
  title={Optimizing neural networks with kronecker-factored approximate curvature},
  author={Martens, James and Grosse, Roger},
  booktitle={International conference on machine learning},
  pages={2408--2417},
  year={2015},
  organization={PMLR}
}

@article{schutt2018schnet,
  title={Schnet--a deep learning architecture for molecules and materials},
  author={Sch{\"u}tt, Kristof T and Sauceda, Huziel E and Kindermans, P-J and Tkatchenko, Alexandre and M{\"u}ller, K-R},
  journal={The Journal of chemical physics},
  volume={148},
  number={24},
  year={2018},
  publisher={AIP Publishing}
}

@article{drissi2024second,
  title={Second-order optimization strategies for neural network quantum states},
  author={Drissi, M and Keeble, JW T and Rozal{\'e}n Sarmiento, J and Rios, A},
  journal={Philosophical Transactions A},
  volume={382},
  number={2275},
  pages={20240057},
  year={2024},
  publisher={The Royal Society}
}

@article{wang2026generalized,
  title={Generalized lanczos method for systematic optimization of neural-network quantum states},
  author={Wang, Jia-Qi and He, Rong-Qiang and Lu, Zhong-Yi},
  journal={Physical Review B},
  volume={113},
  number={8},
  pages={085120},
  year={2026},
  publisher={APS}
}

@article{goldshlager2024kaczmarz,
  title={A Kaczmarz-inspired approach to accelerate the optimization of neural network wavefunctions},
  author={Goldshlager, Gil and Abrahamsen, Nilin and Lin, Lin},
  journal={Journal of Computational Physics},
  volume={516},
  pages={113351},
  year={2024},
  publisher={Elsevier}
}

@article{shokry2025less,
  title={When Less is More: Approximating the Quantum Geometric Tensor with Block Structures},
  author={Shokry, Ahmedeo and Santini, Alessandro and Vicentini, Filippo},
  journal={arXiv preprint arXiv:2510.08430},
  year={2025}
}

@article{ren2019efficient,
  title={Efficient subsampled gauss-newton and natural gradient methods for training neural networks},
  author={Ren, Yi and Goldfarb, Donald},
  journal={arXiv preprint arXiv:1906.02353},
  year={2019}
}

@article{nys2024ab,
  title={Ab-initio variational wave functions for the time-dependent many-electron Schr{\"o}dinger equation},
  author={Nys, Jannes and Pescia, Gabriel and Sinibaldi, Alessandro and Carleo, Giuseppe},
  journal={Nature communications},
  volume={15},
  number={1},
  pages={9404},
  year={2024},
  publisher={Nature Publishing Group UK London}
}

@article{gauvin-ndiaye2025mott-e75, 
  year    = {2025}, 
  title   = {Mott Transition and Volume Law Entanglement with Neural Quantum States}, 
  author  = {Gauvin-Ndiaye, Chloé and Tindall, Joseph and Moreno, Javier Robledo and Georges, Antoine}, 
  journal = {Physical Review Letters}, 
  issn    = {0031-9007}, 
  doi     = {10.1103/physrevlett.134.076502}, 
  pmid    = {40053943}, 
  pages   = {076502}, 
  number  = {7}, 
  volume  = {134}
}

@article{pescia2024message-passing-11d, 
  year    = {2024}, 
  title   = {Message-passing neural quantum states for the homogeneous electron gas}, 
  author  = {Pescia, Gabriel and Nys, Jannes and Kim, Jane and Lovato, Alessandro and Carleo, Giuseppe}, 
  journal = {Physical Review B}, 
  issn    = {2469-9950}, 
  doi     = {10.1103/physrevb.110.035108}, 
  eprint  = {2305.07240}, 
  pages   = {035108}, 
  number  = {3}, 
  volume  = {110}
}

@article{hibat2020recurrent, 
  year    = {2020}, 
  title   = {Recurrent neural network wave functions}, 
  author  = {Hibat-Allah, Mohamed and Ganahl, Martin and Hayward, Lauren E. and Melko, Roger G. and Carrasquilla, Juan}, 
  journal = {Physical Review Research}, 
  doi     = {10.1103/physrevresearch.2.023358}, 
  eprint  = {2002.02973}, 
  pages   = {023358}, 
  number  = {2}, 
  volume  = {2}
}

@article{liu2025efficient-31b, 
  year    = {2025}, 
  title   = {Efficient optimization of variational autoregressive networks with natural gradient}, 
  author  = {Liu, Jing and Tang, Ying and Zhang, Pan}, 
  journal = {Physical Review E}, 
  issn    = {2470-0045}, 
  doi     = {10.1103/physreve.111.025304}, 
  pmid    = {40103066}, 
  eprint  = {2409.20029}, 
  pages   = {025304}, 
  number  = {2}, 
  volume  = {111}
}

@article{sharir2020, 
  year    = {2020}, 
  title   = {Deep Autoregressive Models for the Efficient Variational Simulation of Many-Body Quantum Systems}, 
  author  = {Sharir, Or and Levine, Yoav and Wies, Noam and Carleo, Giuseppe and Shashua, Amnon}, 
  journal = {Physical Review Letters}, 
  issn    = {0031-9007}, 
  doi     = {10.1103/physrevlett.124.020503}, 
  pmid    = {32004039}, 
  eprint  = {1902.04057}, 
  pages   = {020503}, 
  number  = {2}, 
  volume  = {124}
}

@article{amariNaturalGradientWorks1998,
	title = {Natural {Gradient} {Works} {Efficiently} in {Learning}},
	volume = {10},
	issn = {0899-7667, 1530-888X},
	url = {https://direct.mit.edu/neco/article/10/2/251-276/6143},
	doi = {10.1162/089976698300017746},
	language = {en},
	number = {2},
	urldate = {2025-12-03},
	journal = {Neural Computation},
	author = {Amari, Shun-ichi},
	month = feb,
	year = {1998},
	pages = {251--276},
}

@misc{peng2025rwkv7gooseexpressivedynamic,
      title={RWKV-7 "Goose" with Expressive Dynamic State Evolution}, 
      author={Bo Peng and Ruichong Zhang and Daniel Goldstein and Eric Alcaide and Xingjian Du and Haowen Hou and Jiaju Lin and Jiaxing Liu and Janna Lu and William Merrill and Guangyu Song and Kaifeng Tan and Saiteja Utpala and Nathan Wilce and Johan S. Wind and Tianyi Wu and Daniel Wuttke and Christian Zhou-Zheng},
      year={2025},
      eprint={2503.14456},
      archivePrefix={arXiv},
      primaryClass={cs.CL},
      url={https://arxiv.org/abs/2503.14456}, 
}

@misc{sarkar2026evolutionstrategieshyperscale,
      title={Evolution Strategies at the Hyperscale}, 
      author={Bidipta Sarkar and Mattie Fellows and Juan Agustin Duque and Alistair Letcher and Antonio León Villares and Anya Sims and Clarisse Wibault and Dmitry Samsonov and Dylan Cope and Jarek Liesen and Kang Li and Lukas Seier and Theo Wolf and Uljad Berdica and Valentin Mohl and Alexander David Goldie and Aaron Courville and Karin Sevegnani and Shimon Whiteson and Jakob Nicolaus Foerster},
      year={2026},
      eprint={2511.16652},
      archivePrefix={arXiv},
      primaryClass={cs.LG},
      url={https://arxiv.org/abs/2511.16652}, 
}

@article{sorella_green_1998,
	title = {Green {Function} {Monte} {Carlo} with {Stochastic} {Reconfiguration}},
	volume = {80},
	url = {https://link.aps.org/doi/10.1103/PhysRevLett.80.4558},
	doi = {10.1103/PhysRevLett.80.4558},
	number = {20},
	urldate = {2026-03-04},
	journal = {Physical Review Letters},
	publisher = {American Physical Society},
	author = {Sorella, Sandro},
	month = may,
	year = {1998},
	pages = {4558--4561},
}

@InProceedings{pmlr-v107-barr20a,
  title = 	 {Quantum Ground States from Reinforcement Learning},
  author =       {Barr, Ariel and Gispen, Willem and Lamacraft, Austen},
  booktitle = 	 {Proceedings of The First Mathematical and Scientific Machine Learning Conference},
  pages = 	 {635--653},
  year = 	 {2020},
  editor = 	 {Lu, Jianfeng and Ward, Rachel},
  volume = 	 {107},
  series = 	 {Proceedings of Machine Learning Research},
  month = 	 {20--24 Jul},
  publisher =    {PMLR},
  url = 	 {https://proceedings.mlr.press/v107/barr20a.html}
}

@article{wuVariationalBenchmarksQuantum2024,
	title = {Variational benchmarks for quantum many-body problems},
	volume = {386},
	url = {https://www.science.org/doi/10.1126/science.adg9774},
	doi = {10.1126/science.adg9774},
	number = {6719},
	urldate = {2026-02-20},
	journal = {Science},
	publisher = {American Association for the Advancement of Science},
	author = {Wu, Dian and Rossi, Riccardo and Vicentini, Filippo and Astrakhantsev, Nikita and Becca, Federico and Cao, Xiaodong and Carrasquilla, Juan and Ferrari, Francesco and Georges, Antoine and Hibat-Allah, Mohamed and Imada, Masatoshi and Läuchli, Andreas M. and Mazzola, Guglielmo and Mezzacapo, Antonio and Millis, Andrew and Robledo Moreno, Javier and Neupert, Titus and Nomura, Yusuke and Nys, Jannes and Parcollet, Olivier and Pohle, Rico and Romero, Imelda and Schmid, Michael and Silvester, J. Maxwell and Sorella, Sandro and Tocchio, Luca F. and Wang, Lei and White, Steven R. and Wietek, Alexander and Yang, Qi and Yang, Yiqi and Zhang, Shiwei and Carleo, Giuseppe},
	month = oct,
	year = {2024},
	pages = {296--301},
}

@article{jae2025reinforcement-2b9, 
  year    = {2025}, 
  title   = {Reinforcement Learning to Learn Quantum States for Heisenberg Scaling Accuracy}, 
  author  = {Jae, Jeongwoo and Hong, Jeonghoon and Choo, Jinho and Kwon, Yeong‐Dae}, 
  journal = {Advanced Quantum Technologies}, 
  issn    = {2511-9044}, 
  doi     = {10.1002/qute.202500206}, 
  eprint  = {2412.02334}, 
  number  = {10}, 
  volume  = {8}
}

@inproceedings{
chen2022systematic,
title={Systematic improvement of neural network quantum states using Lanczos},
author={Hongwei Chen and Douglas Gerard Hendry and Phillip E Weinberg and Adrian Feiguin},
booktitle={Advances in Neural Information Processing Systems},
editor={Alice H. Oh and Alekh Agarwal and Danielle Belgrave and Kyunghyun Cho},
year={2022},
url={https://openreview.net/forum?id=qZUHvvtbzy}
}

@article{nomura2021dirac-type-d78, 
  year    = {2021}, 
  title   = {Dirac-Type Nodal Spin Liquid Revealed by Refined Quantum Many-Body Solver Using Neural-Network Wave Function, Correlation Ratio, and Level Spectroscopy}, 
  author  = {Nomura, Yusuke and Imada, Masatoshi}, 
  journal = {Physical Review X}, 
  doi     = {10.1103/physrevx.11.031034}, 
  eprint  = {2005.14142}, 
  pages   = {031034}, 
  number  = {3}, 
  volume  = {11}
}

@article{rende2025foundation-7f9, 
  year    = {2025}, 
  title   = {Foundation neural-networks quantum states as a unified Ansatz for multiple hamiltonians}, 
  author  = {Rende, Riccardo and Viteritti, Luciano Loris and Becca, Federico and Scardicchio, Antonello and Laio, Alessandro and Carleo, Giuseppe}, 
  journal = {Nature Communications}, 
  doi     = {10.1038/s41467-025-62098-x}, 
  pmid    = {40764510}, 
  pmcid   = {{PMC}12325958}, 
  eprint  = {2502.09488}, 
  pages   = {7213}, 
  number  = {1}, 
  volume  = {16}
}

@InProceedings{pmlr-v145-gispen22a,
  title = 	 {Ground States of Quantum Many Body Lattice Models via Reinforcement Learning},
  author =       {Gispen, Willem and Lamacraft, Austen},
  booktitle = 	 {Proceedings of the 2nd Mathematical and Scientific Machine Learning Conference},
  pages = 	 {369--385},
  year = 	 {2022},
  editor = 	 {Bruna, Joan and Hesthaven, Jan and Zdeborova, Lenka},
  volume = 	 {145},
  series = 	 {Proceedings of Machine Learning Research},
  month = 	 {16--19 Aug},
  publisher =    {PMLR},
  url = 	 {https://proceedings.mlr.press/v145/gispen22a.html}
}

@book{horn_matrix_2012,
	address = {Cambridge ; New York},
	edition = {2nd ed},
	title = {Matrix analysis},
	isbn = {978-0-521-83940-2},
	publisher = {Cambridge University Press},
	author = {Horn, Roger A. and Johnson, Charles R.},
	year = {2012},
}

@article{wolff1990critical-079, 
  year    = {1990}, 
  title   = {Critical slowing down}, 
  author  = {Wolff, Ulli}, 
  journal = {Nuclear Physics B - Proceedings Supplements}, 
  issn    = {0920-5632}, 
  doi     = {10.1016/0920-5632(90)90224-i}, 
  pages   = {93--102}, 
  volume  = {17}
}

@article{debbio2004critical-247, 
  year    = {2004}, 
  title   = {Critical slowing down of topological modes}, 
  author  = {Debbio, Luigi Del and Manca, Gian Mario and Vicari, Ettore}, 
  journal = {Physics Letters B}, 
  issn    = {0370-2693}, 
  doi     = {10.1016/j.physletb.2004.05.038}, 
  eprint  = {hep-lat/0403001}, 
  pages   = {315--323}, 
  number  = {3-4}, 
  volume  = {594}
}

@article{moss2026double-cec, 
  year    = {2026}, 
  title   = {Double descent: When do neural quantum states generalize?}, 
  author  = {Moss, M. Schuyler and Orfi, Alev and Roth, Christopher and Sengupta, Anirvan M. and Georges, Antoine and Sels, Dries and Dawid, Anna and Valenti, Agnes}, 
  journal = {Physical Review E}, 
  issn    = {2470-0045}, 
  doi     = {10.1103/cwmj-fxr4}, 
  eprint  = {2508.00068}, 
  pages   = {045303}, 
  number  = {4}, 
  volume  = {113}
}

@article{moss2025leveraging-475, 
  year    = {2025}, 
  title   = {Leveraging recurrence in neural network wavefunctions for large-scale simulations of Heisenberg antiferromagnets on the square lattice}, 
  author  = {Moss, M. Schuyler and Wiersema, Roeland and Hibat-Allah, Mohamed and Carrasquilla, Juan and Melko, Roger G.}, 
  journal = {Physical Review B}, 
  issn    = {2469-9950}, 
  doi     = {10.1103/6ccd-wzhz}, 
  pages   = {134450}, 
  number  = {13}, 
  volume  = {112}
}

\doparttoc \faketableofcontents \part{} 
\newpage
\appendix
\addcontentsline{toc}{section}{Appendix}
\part{Appendix} 

\parttoc

\newpage
\section{Dirac Notation}
\label{app:dirac}

This appendix provides a self-contained introduction to the Dirac (bra-ket) notation used throughout the paper.

\textbf{Hilbert space.}
A quantum state lives in a complex vector space $\mathcal{H}$ called a \emph{Hilbert space}, equipped with an inner product. For a system of $N$ spin-$1/2$ particles, the Hilbert space is $\mathcal{H} = (\mathbb{C}^2)^{\otimes N}$, which has dimension $2^N$. A natural orthonormal basis is the \emph{computational} (spin) basis, whose elements $\ket{\vb{s}}$ are indexed by binary spin configurations $\vb{s} \in \{\pm1\}^N$.

\textbf{Ket vectors.}
A quantum state is written as a \emph{ket} $\ket{\psi}$, which is simply a column vector in $\mathcal{H}$. Any state can be expanded in the computational basis,
\begin{equation}
    \ket{\psi} = \sum_{\vb{s}\in\{\pm1\}^N} \psi(\vb{s})\,\ket{\vb{s}},
\end{equation}
where each coefficient $\psi(\vb{s}) \in \mathbb{C}$ is the \emph{amplitude} of configuration $\vb{s}$. The vector $\ket{\psi}$ is therefore completely specified by the mapping $\vb{s} \mapsto \psi(\vb{s})$, which NQS parameterize with a neural network.

\textbf{Bra vectors and inner product.}
The \emph{bra} $\bra{\psi}$ is the conjugate transpose (Hermitian adjoint) of $\ket{\psi}$,
\begin{equation}
    \bra{\psi} = \ket{\psi}^\dagger.
\end{equation}
The \emph{inner product} of two states is written $\braket{\phi}{\psi}$, which evaluates to the complex number $\sum_{\vb{s}} \phi(\vb{s})^* \psi(\vb{s})$. When $\phi = \psi$ this gives the squared norm $\braket{\psi}{\psi} = \sum_{\vb{s}} |\psi(\vb{s})|^2 \geq 0$, with equality only for the zero vector. The basis states are orthonormal:
\begin{equation}
    \braket{\vb{s}}{\vb{s}'} = \delta_{\vb{s},\vb{s}'},
\end{equation}
where $\delta_{\vb{s},\vb{s}'}$ is the Kronecker delta.

\textbf{Completeness relation.}
The basis states form a complete set, meaning the \emph{identity operator} $\I$ can be resolved as
\begin{equation}
    \sum_{\vb{s}} \ket{\vb{s}}\bra{\vb{s}} = \I.
\end{equation}
This is used repeatedly in derivations (e.g.\ Appendix~\ref{app:variational_gradient}) to insert a basis expansion and convert abstract operator equations into sums over spin configurations.

\textbf{Operators and matrix elements.}
A quantum \emph{operator} $\hat{A}$ is a linear map $\mathcal{H}\to\mathcal{H}$, represented as a $2^N \times 2^N$ matrix in the computational basis. The entry at row $\vb{s}$ and column $\vb{s}'$ is the \emph{matrix element}
\begin{equation}
    \mel{\vb{s}}{\hat{A}}{\vb{s}'} = \bra{\vb{s}}\!\left(\hat{A}\ket{\vb{s}'}\right) = (\vb{A})_{\vb{s},\vb{s}'}.
\end{equation}
An operator is \emph{Hermitian} (or self-adjoint) if $\hat{A} = \hat{A}^\dagger$, which implies real eigenvalues. The Hamiltonian $\hat{H}$ is always Hermitian because energy is a real-valued observable.

\textbf{Expectation value.}
The expected value of an observable $\hat{A}$ in state $\ket{\psi}$ is
\begin{equation}
    \langle \hat{A} \rangle_\psi := \frac{\mel{\psi}{\hat{A}}{\psi}}{\braket{\psi}{\psi}}.
\end{equation}
For the Hamiltonian, this gives the variational energy $E[\psi] = \mel{\psi}{\hat{H}}{\psi}/\braket{\psi}{\psi}$ that NQS minimizes. When $\ket{\psi}$ is normalized ($\braket{\psi}{\psi}=1$), the denominator can be dropped.

\textbf{Summary table.}
\begin{center}
\begin{tabular}{lll}
\toprule
\textbf{Dirac symbol} & \textbf{Linear-algebra equivalent} & \textbf{Meaning in this paper} \\
\midrule
$\ket{\psi}$ & column vector in $\mathbb{C}^{2^N}$ & quantum state / wavefunction \\
$\bra{\psi}$ & conjugate row vector & dual / adjoint state \\
$\braket{\phi}{\psi}$ & $\vb*{\phi}^\dagger \vb*{\psi}$ (dot product) & inner product; overlap of two states \\
$\braket{\psi}{\psi}$ & $\|\vb*{\psi}\|^2$ & squared norm; $= 1$ for normalized states \\
$\ket{\vb{s}}$ & standard basis vector $e_{\vb{s}}$ & spin-configuration basis state \\
$\mel{\vb{s}}{\hat{H}}{\vb{s}'}$ & $(\vb{H})_{\vb{s},\vb{s}'}$ matrix entry & Hamiltonian matrix element \\
$\mel{\psi}{\hat{H}}{\psi}$ & $\vb*{\psi}^\dagger \vb{H} \vb*{\psi}$ (quadratic form) & (unnormalized) variational energy \\
$\sum_{\vb{s}}\ket{\vb{s}}\bra{\vb{s}}$ & $\I_{2^N}$ (identity matrix) & completeness / resolution of identity \\
\bottomrule
\end{tabular}
\end{center}

\section{Mathematical Statements}

\subsection{Proof of Proposition \ref{prop:pg_nqs}}
\label{app:proof_of_pg_nqs}
\pgnqsprop*

\begin{proof}
We follow the proof strategy of the policy gradient theorem \citep{agarwal2021reinforcement} to connect the standard NQS/VMC gradient estimator with REINFORCE-style updates when the probability distribution is parameterized autoregressively.

Differentiating gives
\begin{align}
    \grad_{\boldsymbol{\theta}} E[\psi_{\boldsymbol{\theta}}]
    &= \grad_{\boldsymbol{\theta}} \sum_{\vb{s}} \mathcal{P}_{\boldsymbol{\theta}}(\vb{s})\,E^{\mathrm{loc}}_{\boldsymbol{\theta}}(\vb{s})\\
    &= \sum_{\vb{s}} \grad_{\boldsymbol{\theta}} \mathcal{P}_{\boldsymbol{\theta}}(\vb{s})\,E^{\mathrm{loc}}_{\boldsymbol{\theta}}(\vb{s})
    \;+\;\sum_{\vb{s}} \mathcal{P}_{\boldsymbol{\theta}}(\vb{s})\,\grad_{\boldsymbol{\theta}} E^{\mathrm{loc}}_{\boldsymbol{\theta}}(\vb{s})\\
    &= \sum_{\vb{s}} \mathcal{P}_{\boldsymbol{\theta}}(\vb{s})\,E^{\mathrm{loc}}_{\boldsymbol{\theta}}(\vb{s})\grad_{\boldsymbol{\theta}} \log \mathcal{P}_{\boldsymbol{\theta}}(\vb{s})
    \;+\;\sum_{\vb{s}} \mathcal{P}_{\boldsymbol{\theta}}(\vb{s})\,\grad_{\boldsymbol{\theta}} E^{\mathrm{loc}}_{\boldsymbol{\theta}}(\vb{s})\\
    &= \E_{\vb{s}\sim\mathcal{P}_{\boldsymbol{\theta}}}\!\bqty{E^{\mathrm{loc}}_{\boldsymbol{\theta}}(\vb{s})\,\grad_{\boldsymbol{\theta}} \log \mathcal{P}_{\boldsymbol{\theta}}(\vb{s})}
    \;+\;\E_{\vb{s}\sim\mathcal{P}_{\boldsymbol{\theta}}}\!\bqty{\grad_{\boldsymbol{\theta}} E^{\mathrm{loc}}_{\boldsymbol{\theta}}(\vb{s})}.
    \label{eq:energy-split}
\end{align}
We now analyze the second term. Define the log-derivative observable
\[
\vb{O}(\vb{s}):=\grad_{\boldsymbol{\theta}}\log\psi_{\boldsymbol{\theta}}(\vb{s}),
\]
then
\begin{align}\label{eq:second_term}
    \E_{\vb{s}\sim\mathcal{P}_{\boldsymbol{\theta}}}\!\bqty{\grad_{\boldsymbol{\theta}} E^{\mathrm{loc}}_{\boldsymbol{\theta}}(\vb{s})}
    &= \E_{\vb{s}\sim\mathcal{P}_{\boldsymbol{\theta}}}\!\bqty{\grad_{\boldsymbol{\theta}} \sum_{\vb{s}'}\frac{\psi_{\boldsymbol{\theta}}(\vb{s}')}{\psi_{\boldsymbol{\theta}}(\vb{s})}\mel{\vb{s}}{\hat{H}}{\vb{s}'}}\\
    &= \E_{\vb{s}\sim\mathcal{P}_{\boldsymbol{\theta}}}\!\bqty{\sum_{\vb{s}'}\mel{\vb{s}}{\hat{H}}{\vb{s}'}\frac{\pqty{\psi_{\boldsymbol{\theta}}(\vb{s})\grad_{\boldsymbol{\theta}}\psi_{\boldsymbol{\theta}}(\vb{s}') -\grad_{\boldsymbol{\theta}}\psi_{\boldsymbol{\theta}}(\vb{s})\psi_{\boldsymbol{\theta}}(\vb{s}')}}{\psi_{\boldsymbol{\theta}}(\vb{s})^2}}\\
    &= \E_{\vb{s}\sim\mathcal{P}_{\boldsymbol{\theta}}}\!\bqty{\sum_{\vb{s}'}\mel{\vb{s}}{\hat{H}}{\vb{s}'}\,
    \frac{\psi_{\boldsymbol{\theta}}(\vb{s}')}{\psi_{\boldsymbol{\theta}}(\vb{s})}\pqty{\vb{O}(\vb{s}')-\vb{O}(\vb{s})}}\\
    &= \sum_{\vb{s}} \mathcal{P}_{\boldsymbol{\theta}}(\vb{s})\sum_{\vb{s}'} \hat{H}_{\vb{s},\vb{s}'}\frac{\psi_{\boldsymbol{\theta}}(\vb{s}')}{\psi_{\boldsymbol{\theta}}(\vb{s})}
    \pqty{\vb{O}(\vb{s}')-\vb{O}(\vb{s})},
\label{eq:pathwise-1}
\end{align}
where $\hat{H}_{\vb{s},\vb{s}'}=\mel{\vb{s}}{\hat{H}}{\vb{s}'}$,
which gives the double-sum form
\begin{equation}
\E_{\vb{s}\sim\mathcal{P}_{\boldsymbol{\theta}}}\!\bqty{\grad_{\boldsymbol{\theta}} E^{\mathrm{loc}}_{\boldsymbol{\theta}}(\vb{s})}
=
\frac{1}{Z_{\boldsymbol{\theta}}}\sum_{\vb{s},\vb{s}'}
\psi_{\boldsymbol{\theta}}(\vb{s})\psi_{\boldsymbol{\theta}}(\vb{s}')\,\hat{H}_{\vb{s},\vb{s}'}
\pqty{\vb{O}(\vb{s}')-\vb{O}(\vb{s})}.
\label{eq:pathwise-2}
\end{equation}
At this point, the only remaining question is whether the second term in~\eqref{eq:pathwise-2} vanishes. This happens when the Hamiltonian is stoquastic in the chosen basis. In that case, its matrix elements are real and Hermiticity implies $\hat{H}_{\vb{s},\vb{s}'}=\hat{H}_{\vb{s}',\vb{s}}$. Moreover, according to Perron-Frobenius  \citep{horn_matrix_2012}, all the components of the ground state wavefunction of a stoquastic Hamiltonian are real and strictly positive, so $\psi_{\boldsymbol{\theta}}(\vb{s}) = \abs{\psi_{\boldsymbol{\theta}}(\vb{s})}$ and the log-derivative observable $\vb{O}(\vb{s})$ is real-valued.

Under this assumption, the factor $\psi_{\boldsymbol{\theta}}(\vb{s})\psi_{\boldsymbol{\theta}}(\vb{s}')\hat{H}_{\vb{s},\vb{s}'}$ in~\eqref{eq:pathwise-2} is symmetric under exchanging $\vb{s}$ and $\vb{s}'$, whereas the difference $\vb{O}(\vb{s}')-\vb{O}(\vb{s})$ is antisymmetric. Therefore, each term in the double sum cancels with the term obtained by swapping $\vb{s}$ and $\vb{s}'$, and the whole sum is zero. Hence,
\begin{equation}
\E_{\vb{s}\sim\mathcal{P}_{\boldsymbol{\theta}}}\!\bqty{\grad_{\boldsymbol{\theta}} E^{\mathrm{loc}}_{\boldsymbol{\theta}}(\vb{s})}=0.
\end{equation}
Substituting back into~\eqref{eq:energy-split} yields the pure score-function (policy-gradient) estimator:
\begin{equation}
\grad_{\boldsymbol{\theta}} E[\psi_{\boldsymbol{\theta}}]
=
\E_{\vb{s}\sim\mathcal{P}_{\boldsymbol{\theta}}}\!\bqty{E^{\mathrm{loc}}_{\boldsymbol{\theta}}(\vb{s})\,\grad_{\boldsymbol{\theta}}\log \mathcal{P}_{\boldsymbol{\theta}}(\vb{s})}.
\label{eq:score}
\end{equation}

Moreover, for any constant $c$ that doesn't depend on $\vb{s}$,
\begin{align}
    \E_{\vb{s}\sim\mathcal{P}_{\boldsymbol{\theta}}}\!\bqty{c \grad_{\boldsymbol{\theta}} \log \mathcal{P}_{\boldsymbol{\theta}}(\vb{s})} &=  c \sum_{\vb{s}} \mathcal{P}_{\boldsymbol{\theta}}(\vb{s})\grad_{\boldsymbol{\theta}} \log \mathcal{P}_{\boldsymbol{\theta}}(\vb{s}) \\
    &= c \sum_{\vb{s}}\grad_{\boldsymbol{\theta}} \mathcal{P}_{\boldsymbol{\theta}}(\vb{s})\\
    &= c \grad_{\boldsymbol{\theta}} \sum_{\vb{s}} \mathcal{P}_{\boldsymbol{\theta}}(\vb{s})\\
    &= c \grad_{\boldsymbol{\theta}} 1 = 0.
\end{align}
So we may
subtract any constant baseline without bias \citep{agarwal2021reinforcement}. Choosing the baseline as the global energy
$E[\psi_{\boldsymbol{\theta}}]=\E_{\vb{s}\sim\mathcal{P}_{\boldsymbol{\theta}}}[E^{\mathrm{loc}}_{\boldsymbol{\theta}}(\vb{s})]$ gives the form
\begin{equation}
\grad_{\boldsymbol{\theta}} E[\psi_{\boldsymbol{\theta}}]
=
\E_{\vb{s}\sim\mathcal{P}_{\boldsymbol{\theta}}}\!\bqty{\pqty{E^{\mathrm{loc}}_{\boldsymbol{\theta}}(\mathbf{s}) - \mathbb{E}_{\mathbf{s}\sim\mathcal{P}_{\boldsymbol{\theta}}}[E^{\mathrm{loc}}_{\boldsymbol{\theta}}(\mathbf{s})]}\,\grad_{\boldsymbol{\theta}}\log \mathcal{P}_{\boldsymbol{\theta}}(\vb{s})},
\label{eq:reinforce}
\end{equation}
which is analogous to the Advantage form of the REINFORCE estimator \citep{agarwal2021reinforcement}. 
\end{proof}

\subsection{Variational Gradient Derivation}\label{app:variational_gradient}

In this section, we show how to derive the well known VMC-gradient in Eq.~\eqref{eq:vmc} \citep{Carleo_2017} from Eq.~\eqref{eq:energy_expectation}. We start from the general expression for the energy, valid for any (not necessarily normalized) state $|\psi_{\boldsymbol{\theta}}\rangle$:
\begin{equation}
    E[\psi_{\boldsymbol{\theta}}] = \frac{\langle\psi_{\boldsymbol{\theta}}|\hat{H}|\psi_{\boldsymbol{\theta}}\rangle}{\langle\psi_{\boldsymbol{\theta}}|\psi_{\boldsymbol{\theta}}\rangle}.
\end{equation}
Differentiating with respect to $\theta_k$ using the quotient rule and the Hermiticity of $\hat{H}$ gives:
\begin{equation}
\partial_{\theta_k} E[\psi_{\boldsymbol{\theta}}] = \frac{2\operatorname{Re}\left\{(\partial_{\theta_k}\langle\psi_{\boldsymbol{\theta}}|)\hat{H}|\psi_{\boldsymbol{\theta}}\rangle\right\}}{\langle\psi_{\boldsymbol{\theta}}|\psi_{\boldsymbol{\theta}}\rangle} - E[\psi_{\boldsymbol{\theta}}]\frac{2\operatorname{Re}\left\{(\partial_{\theta_k}\langle\psi_{\boldsymbol{\theta}}|)|\psi_{\boldsymbol{\theta}}\rangle\right\}}{\langle\psi_{\boldsymbol{\theta}}|\psi_{\boldsymbol{\theta}}\rangle}.
\end{equation}
Since the final result is norm-invariant, we now choose the convenient normalization $\langle\psi_{\boldsymbol{\theta}}|\psi_{\boldsymbol{\theta}}\rangle = 1$, which simplifies the expression to:
\begin{equation}
\partial_{\theta_k} E[\psi_{\boldsymbol{\theta}}] = 2\operatorname{Re}\left\{(\partial_{\theta_k}\langle\psi_{\boldsymbol{\theta}}|)\bigl(\hat{H} - E[\psi_{\boldsymbol{\theta}}]\bigr)|\psi_{\boldsymbol{\theta}}\rangle\right\}.
\end{equation}
We insert the completeness relation $\sum_{\mathbf{s}}|\mathbf{s}\rangle\langle\mathbf{s}|=\I$ twice and use $\mathcal{P}_{\boldsymbol{\theta}}(\mathbf{s}) = |\psi_{\boldsymbol{\theta}}(\mathbf{s})|^2$:
\begin{align}
(\partial_{\theta_k}\langle\psi_{\boldsymbol{\theta}}|)\bigl(\hat{H} - E[\psi_{\boldsymbol{\theta}}]\bigr)|\psi_{\boldsymbol{\theta}}\rangle
&= \sum_{\mathbf{s}} (\partial_{\theta_k}\langle\psi_{\boldsymbol{\theta}}|)|\mathbf{s}\rangle\,\langle\mathbf{s}|\bigl(\hat{H} - E[\psi_{\boldsymbol{\theta}}]\bigr)|\psi_{\boldsymbol{\theta}}\rangle\\
&= \sum_{\mathbf{s}} \partial_{\theta_k}\psi^*_{\boldsymbol{\theta}}(\mathbf{s})\left(\sum_{\mathbf{s}'}\langle\mathbf{s}|\hat{H}|\mathbf{s}'\rangle\langle\mathbf{s}'|\psi_{\boldsymbol{\theta}}\rangle - E[\psi_{\boldsymbol{\theta}}]\,\psi_{\boldsymbol{\theta}}(\mathbf{s})\right)\\
&= \sum_{\mathbf{s}} \psi^*_{\boldsymbol{\theta}}(\mathbf{s})O_k^*(\mathbf{s})\left(\sum_{\mathbf{s}'}\hat{H}_{\mathbf{s},\mathbf{s}'}\psi_{\boldsymbol{\theta}}(\mathbf{s}') - E[\psi_{\boldsymbol{\theta}}]\,\psi_{\boldsymbol{\theta}}(\mathbf{s})\right)\\
&= \sum_{\mathbf{s}} |\psi_{\boldsymbol{\theta}}(\mathbf{s})|^2\, O^*_k(\mathbf{s})\left(E^{\mathrm{loc}}_{\boldsymbol{\theta}}(\mathbf{s}) - E[\psi_{\boldsymbol{\theta}}]\right)\\
&= \mathbb{E}_{\mathbf{s} \sim \mathcal{P}_{\boldsymbol{\theta}}}\!\left[O^*_k(\mathbf{s})\Bigl(E^{\mathrm{loc}}_{\boldsymbol{\theta}}(\mathbf{s}) - \mathbb{E}_{\mathbf{s}\sim\mathcal{P}_{\boldsymbol{\theta}}}[E^{\mathrm{loc}}_{\boldsymbol{\theta}}(\mathbf{s})]\Bigr)\right],\\
\end{align}
where in the last step we used that $E[\psi_{\boldsymbol{\theta}}] = \mathbb{E}_{\mathbf{s}\sim\mathcal{P}_{\boldsymbol{\theta}}}[E^{\mathrm{loc}}_{\boldsymbol{\theta}}(\mathbf{s})]$. Therefore,
\begin{equation}
\partial_{\theta_k} E[\psi_{\boldsymbol{\theta}}] = 2\operatorname{Re}\left\{\mathbb{E}_{\mathbf{s} \sim \mathcal{P}_{\boldsymbol{\theta}}}\!\left[O^*_k(\mathbf{s})\Bigl(E^{\mathrm{loc}}_{\boldsymbol{\theta}}(\mathbf{s}) - \mathbb{E}_{\mathbf{s}\sim\mathcal{P}_{\boldsymbol{\theta}}}[E^{\mathrm{loc}}_{\boldsymbol{\theta}}(\mathbf{s})]\Bigr)\right]\right\}.
\end{equation}

\subsection{Variational Gradient Decomposition}
\label{app:vmc_decomposition}
For the development of a practical algorithm we can decompose the VMC gradient into its probability and phase components as follows. Starting from the standard VMC/NQS gradient estimator we can write
\begin{equation}
    \grad_{\boldsymbol{\theta}} E[\psi_{\boldsymbol{\theta}}]
    =
    2\,\Re\Bqty{
    \E_{\vb{s}\sim\mathcal{P}_{\boldsymbol{\theta}}}\!\bqty{
    \pqty{E^{\mathrm{loc}}_{\boldsymbol{\theta}}(\mathbf{s}) - \mathbb{E}_{\mathbf{s}\sim\mathcal{P}_{\boldsymbol{\theta}}}[E^{\mathrm{loc}}_{\boldsymbol{\theta}}(\mathbf{s})]}\,\vb{O}(\vb{s})^\ast
    }}.
\end{equation}
Recall that $\ast$ is the complex conjugate, and
\begin{equation}
    E^{\mathrm{loc}}_{\boldsymbol{\theta}}(\vb{s}) = \Re\Bqty{E^{\mathrm{loc}}_{\boldsymbol{\theta}}(\vb{s})} + i\cdot \Im\Bqty{E^{\mathrm{loc}}_{\boldsymbol{\theta}}(\vb{s})}, \quad \vb{O}(\vb{s}) = \Re\Bqty{\vb{O}(\vb{s})} + i\cdot\Im\Bqty{\vb{O}(\vb{s})}.
\end{equation}
For convenience, we denote $E_R := \Re\Bqty{E^{\mathrm{loc}}_{\boldsymbol{\theta}}(\vb{s})}$ and $E_I := \Im\Bqty{E^{\mathrm{loc}}_{\boldsymbol{\theta}}(\vb{s})}$ (analogously for $\vb{O}(\vb{s})$).  We can now write
\begin{align}
    \grad_{\boldsymbol{\theta}} E[\psi_{\boldsymbol{\theta}}] =&\, 2\,\Re\Bqty{
    \E_{\vb{s}\sim\mathcal{P}_{\boldsymbol{\theta}}}\bqty{\pqty{(E_R - \mathbb{E}_{\mathbf{s}\sim\mathcal{P}_{\boldsymbol{\theta}}}[E^{\mathrm{loc}}_{\boldsymbol{\theta}}(\mathbf{s})]) + iE_I}\pqty{O_R - iO_I}}}\\
    =&\, 2\,\Re\{
    \E_{\vb{s}\sim\mathcal{P}_{\boldsymbol{\theta}}}[(E_R - \mathbb{E}_{\mathbf{s}\sim\mathcal{P}_{\boldsymbol{\theta}}}[E^{\mathrm{loc}}_{\boldsymbol{\theta}}(\mathbf{s})])O_R \\
    &\,-i(E_R - \mathbb{E}_{\mathbf{s}\sim\mathcal{P}_{\boldsymbol{\theta}}}[E^{\mathrm{loc}}_{\boldsymbol{\theta}}(\mathbf{s})])O_I + iE_I O_R + E_I O_I]\}\\
    =&\, 2\,\Re\Bqty{
    \E_{\vb{s}\sim\mathcal{P}_{\boldsymbol{\theta}}}\bqty{(E_R - \mathbb{E}_{\mathbf{s}\sim\mathcal{P}_{\boldsymbol{\theta}}}[E^{\mathrm{loc}}_{\boldsymbol{\theta}}(\mathbf{s})])O_R}} + 2\,\Re\Bqty{
    \E_{\vb{s}\sim\mathcal{P}_{\boldsymbol{\theta}}}\bqty{E_I O_I}},
    \label{eq:complex-vmc-gradient}
\end{align}
which shows exactly how to update our model, when we parameterize the real and imaginary parts of the amplitude separately. Moreover, we can subtract a constant baseline, $c$, from the second term (corresponding to the imaginary parts of our parameterization) to reduce the variance:
\begin{align}
    2\,\Re\Bqty{
    \E_{\vb{s}\sim\mathcal{P}_{\boldsymbol{\theta}}}\bqty{E_I (O_I - c)}} &=  2\,\Re\Bqty{
    \E_{\vb{s}\sim\mathcal{P}_{\boldsymbol{\theta}}}\bqty{E_I O_I}} -  2\,\Re\Bqty{
    \E_{\vb{s}\sim\mathcal{P}_{\boldsymbol{\theta}}}\bqty{E_I  c}}\\
    &= 2\,\Re\Bqty{
    \E_{\vb{s}\sim\mathcal{P}_{\boldsymbol{\theta}}}\bqty{E_I O_I}} -  2c\,\Re\Bqty{
    \E_{\vb{s}\sim\mathcal{P}_{\boldsymbol{\theta}}}\bqty{E_I}} \\
    &= 2\,\Re\Bqty{
    \E_{\vb{s}\sim\mathcal{P}_{\boldsymbol{\theta}}}\bqty{E_I O_I}}, \label{eq:hermicity}
\end{align}
where line \ref{eq:hermicity} comes from the reality of expected values under Hermitian operators. So like before, we may subtract a baseline to reduce variance. 

To find the variance-minimizing constant, consider the one-sample estimator for the
imaginary contribution to the gradient
\begin{equation}
    g_I(\vb{s};c) := 2\,E_I(\vb{s})\pqty{O_I(\vb{s})-c}.
\end{equation}
Since $\E_{\vb{s}\sim\mathcal{P}_{\boldsymbol{\theta}}}[E_I(\vb{s})]=0$, subtracting any constant
$c$ leaves the estimator unbiased:
\begin{equation}
    \E[g_I(\vb{s};c)] = 2\,\E\!\left[E_I(\vb{s})\,O_I(\vb{s})\right].
\end{equation}
Therefore, the optimal constant is the one that minimizes
$\mathbbm{V}(g_I(\vb{s};c))$, or equivalently its second moment:
\begin{equation}
    \argmin_c \; \E\!\left[E_I(\vb{s})^2\pqty{O_I(\vb{s})-c}^2\right].
\end{equation}
Differentiating with respect to $c$ and setting the derivative to zero gives
\begin{align}
    0
    &= \frac{\partial}{\partial c}\E\left[E_I(\vb{s})^2\pqty{O_I(\vb{s})-c}^2\right] \\
    &= -2\,\E\left[E_I(\vb{s})^2\pqty{O_I(\vb{s})-c}\right],
\end{align}
hence the variance-minimizing baseline is
\begin{equation}
    c^\star =\frac{\E\!\left[E_I(\vb{s})^2\,O_I(\vb{s})\right]}
         {\E \left[E_I(\vb{s})^2\right]}.
\end{equation}
If $O_I(\vb{s})$ is vector-valued, this formula is applied coordinate-wise. We tried using this minimizer in practice, but its implementation requires adding an extra forward and backward pass, making it significantly slower than the vanilla gradient expression.

\subsection{Exact Complex-Amplitude Decomposition}
\label{app:exact_complex_decomposition}

We now derive an exact decomposition of the energy difference between two
normalized complex wavefunctions. Let
$\ket{\psi_{\boldsymbol{\theta}}}$ be the reference state and
$\ket{\psi_{\boldsymbol{\theta}'}}$ be the candidate state. We choose the global
phase of $\ket{\psi_{\boldsymbol{\theta}'}}$ so that
\begin{equation}
    \braket{\psi_{\boldsymbol{\theta}}}{\psi_{\boldsymbol{\theta}'}}
    =
    \abs{
    \braket{\psi_{\boldsymbol{\theta}}}{\psi_{\boldsymbol{\theta}'}}
    }
    =
    \sqrt{
    1-\mathcal I(
    \psi_{\boldsymbol{\theta}},
    \psi_{\boldsymbol{\theta}'}
    )
    }.
    \label{eq:parallel_transport_gauge}
\end{equation}
Define
\begin{equation}
    r_{\boldsymbol{\theta}';\boldsymbol{\theta}}(\vb{s})
    :=
    \frac{
    \mathcal P_{\boldsymbol{\theta}'}(\vb{s})
    }{
    \mathcal P_{\boldsymbol{\theta}}(\vb{s})
    },
    \qquad
    z_{\boldsymbol{\theta}';\boldsymbol{\theta}}(\vb{s})
    :=
    \frac{
    \psi_{\boldsymbol{\theta}'}(\vb{s})
    }{
    \psi_{\boldsymbol{\theta}}(\vb{s})
    }.
\end{equation}
Writing
\[
    z_{\boldsymbol{\theta}';\boldsymbol{\theta}}(\vb{s})
    =
    \sqrt{
    r_{\boldsymbol{\theta}';\boldsymbol{\theta}}(\vb{s})
    }
    e^{i\alpha_{\boldsymbol{\theta}';\boldsymbol{\theta}}(\vb{s})},
\]
we take the wrapped phase difference
\begin{equation}
    \alpha_{\boldsymbol{\theta}';\boldsymbol{\theta}}(\vb{s})
    =
    \atanTwo
    \left(
    \sin\left(
    \arg\psi_{\boldsymbol{\theta}'}(\vb{s})
    -
    \arg\psi_{\boldsymbol{\theta}}(\vb{s})
    \right),
    \cos\left(
    \arg\psi_{\boldsymbol{\theta}'}(\vb{s})
    -
    \arg\psi_{\boldsymbol{\theta}}(\vb{s})
    \right)
    \right).
\end{equation}
For later use, define the PWO phase increment
\begin{equation}
     \phi_{\boldsymbol{\theta}';\boldsymbol{\theta}}(\vb{s})
    :=
    2\alpha_{\boldsymbol{\theta}';\boldsymbol{\theta}}(\vb{s}).
    \label{eq:pwo_phase_increment_definition}
\end{equation}
The factor of two appears because the VMC phase gradient in
Eq.~\eqref{eq:complex-vmc-gradient} contains
$2A_{\boldsymbol{\theta}}^{\mathrm{I}}\grad\arg\psi_{\boldsymbol{\theta}}$.
\\
\begin{theorem}[Exact complex-amplitude decomposition]
\label{thm:exact-complex-decomposition}
Let
\[
    \hat K_{\boldsymbol{\theta}}
    :=
    \hat H-E[\psi_{\boldsymbol{\theta}}]\I,
    \qquad
    \ket{\eta}
    :=
    \ket{\psi_{\boldsymbol{\theta}'}}
    -
    \ket{\psi_{\boldsymbol{\theta}}},
\]
and let
\[
    \Delta E^{\mathrm{loc}}_{\boldsymbol{\theta}}
    :=
    E^{\mathrm{loc}}_{\boldsymbol{\theta}}
    -
    E[\psi_{\boldsymbol{\theta}}]
    =
    A_{\boldsymbol{\theta}}^{\mathrm{R}}+iA_{\boldsymbol{\theta}}^{\mathrm{I}}.
\]
Then
\begin{align}
    E[\psi_{\boldsymbol{\theta}'}]
    -
    E[\psi_{\boldsymbol{\theta}}]
    &=
    2\Re
    \E_{\vb{s}\sim\mathcal P_{\boldsymbol{\theta}}}
    \left[
    z_{\boldsymbol{\theta}';\boldsymbol{\theta}}(\vb{s})^*
    \Delta E^{\mathrm{loc}}_{\boldsymbol{\theta}}(\vb{s})
    \right]
    +
    \mel{\eta}{\hat K_{\boldsymbol{\theta}}}{\eta}.
    \label{eq:exact-complex-decomposition}
\end{align}
Consequently,
\begin{align}
    E[\psi_{\boldsymbol{\theta}'}]
    -
    E[\psi_{\boldsymbol{\theta}}]
    &\le
    2\E_{\vb{s}\sim\mathcal P_{\boldsymbol{\theta}}}
    \left[
    \sqrt r
    \cos\alpha
    A_{\boldsymbol{\theta}}^{\mathrm{R}}
    \right]
    +
    2\E_{\vb{s}\sim\mathcal P_{\boldsymbol{\theta}}}
    \left[
    \sqrt r
    \sin\alpha
    A_{\boldsymbol{\theta}}^{\mathrm{I}}
    \right]
     \\
    &\quad+
    2
    \norm{
    \hat H-E[\psi_{\boldsymbol{\theta}}]\I
    }_{\infty}
    \left(
    1-
    \sqrt{
    1-\mathcal I(
    \psi_{\boldsymbol{\theta}},
    \psi_{\boldsymbol{\theta}'}
    )
    }
    \right),
    \label{eq:exact-complex-upper-bound}
\end{align}
where, for readability, we have suppressed the
$(\boldsymbol{\theta}';\boldsymbol{\theta})$ subscripts on $r$ and $\alpha$.
\end{theorem}

\begin{proof}
Since
$\mel{\psi_{\boldsymbol{\theta}}}{\hat K_{\boldsymbol{\theta}}}
{\psi_{\boldsymbol{\theta}}}=0$, expanding around the reference state gives
\begin{align}
    E[\psi_{\boldsymbol{\theta}'}]
    -
    E[\psi_{\boldsymbol{\theta}}]
    &=
    \mel{\psi_{\boldsymbol{\theta}'}}{\hat K_{\boldsymbol{\theta}}}
    {\psi_{\boldsymbol{\theta}'}}  \\
    &=
    2\Re
    \mel{\eta}{\hat K_{\boldsymbol{\theta}}}
    {\psi_{\boldsymbol{\theta}}}
    +
    \mel{\eta}{\hat K_{\boldsymbol{\theta}}}{\eta}.
\end{align}
Moreover,
$\eta(\vb{s})=(z_{\boldsymbol{\theta}';\boldsymbol{\theta}}(\vb{s})-1)
\psi_{\boldsymbol{\theta}}(\vb{s})$, and
\[
    \mel{\vb{s}}{\hat K_{\boldsymbol{\theta}}}
    {\psi_{\boldsymbol{\theta}}}
    =
    \psi_{\boldsymbol{\theta}}(\vb{s})
    \Delta E^{\mathrm{loc}}_{\boldsymbol{\theta}}(\vb{s}).
\]
Therefore,
\begin{align}
    \mel{\eta}{\hat K_{\boldsymbol{\theta}}}
    {\psi_{\boldsymbol{\theta}}}
    &=
    \E_{\vb{s}\sim\mathcal P_{\boldsymbol{\theta}}}
    \left[
    \left(
    z_{\boldsymbol{\theta}';\boldsymbol{\theta}}(\vb{s})^*
    -1
    \right)
    \Delta E^{\mathrm{loc}}_{\boldsymbol{\theta}}(\vb{s})
    \right]  \\
    &=
    \E_{\vb{s}\sim\mathcal P_{\boldsymbol{\theta}}}
    \left[
    z_{\boldsymbol{\theta}';\boldsymbol{\theta}}(\vb{s})^*
    \Delta E^{\mathrm{loc}}_{\boldsymbol{\theta}}(\vb{s})
    \right],
\end{align}
because
$\E_{\vb{s}\sim\mathcal P_{\boldsymbol{\theta}}}
[
\Delta E^{\mathrm{loc}}_{\boldsymbol{\theta}}(\vb{s})
]
=0$.
The residual satisfies
\begin{align}
    \mel{\eta}{\hat K_{\boldsymbol{\theta}}}{\eta}
    &\le
    \abs{
    \mel{\eta}{\hat K_{\boldsymbol{\theta}}}{\eta}
    }
    \le
    \norm{\hat K_{\boldsymbol{\theta}}}_{\infty}
    \norm{\eta}^2.
\end{align}
By the phase convention in Eq.~\eqref{eq:parallel_transport_gauge},
\begin{align}
    \norm{\eta}^2
    &=
    2-
    2\Re
    \braket{\psi_{\boldsymbol{\theta}}}{\psi_{\boldsymbol{\theta}'}}
    =
    2\left(
    1-
    \sqrt{
    1-\mathcal I(
    \psi_{\boldsymbol{\theta}},
    \psi_{\boldsymbol{\theta}'}
    )
    }
    \right).
\end{align}
Finally,
\begin{equation}
    \Re
    \left[
    z^*
    \Delta E^{\mathrm{loc}}_{\boldsymbol{\theta}}
    \right]
    =
    \sqrt r
    \left(
    \cos\alpha\,A_{\boldsymbol{\theta}}^{\mathrm{R}}
    +
    \sin\alpha\,A_{\boldsymbol{\theta}}^{\mathrm{I}}
    \right),
\end{equation}
which proves the result.
\end{proof}

\subsection{Amplitude and Phase Trust Regions Control Infidelity}
\label{app:amplitude_phase_trust_region_infidelity}

The bound above, in the Conservative Policy Iteration (CPI) style, depends on the infidelity between the current and
candidate wavefunctions. We now show that separate trust regions on the Born
distribution and the phase imply such an infidelity trust region.

\begin{lemma}[Amplitude and phase trust regions imply an infidelity trust region]
\label{lem:amplitude_phase_controls_infidelity}
Let $0\le\epsilon<1$ and $0\le\delta\le\pi/2$. Assume that we have no measure zero events and that for every
configuration with $\mathcal{P}_{\boldsymbol{\theta}}(\vb{s})>0$,
\begin{equation}
    1-\epsilon
    \le
     r_{\boldsymbol{\theta}';\boldsymbol{\theta}}(\vb{s})
    \le
    1+\epsilon,
    \label{eq:amplitude_trust_region_population}
\end{equation}
and that the phases satisfy, up to a global phase,
\begin{equation}
    \left|
    \alpha_{\boldsymbol{\theta}';\boldsymbol{\theta}}(\vb{s})\right|
    \le
    \delta/2.
    \label{eq:phase_trust_region_population}
\end{equation}
Then
\begin{equation}
    \mathcal I(\psi_{\boldsymbol{\theta}},\psi_{\boldsymbol{\theta}'})
    \le
    1-
    \cos^2\pqty{\frac{\delta}{2}}
    \frac{1+\sqrt{1-\epsilon^2}}{2}.
    \label{eq:infidelity_bound_epsilon_delta}
\end{equation}
\end{lemma}

\begin{proof}
The overlap satisfies
\begin{align}
    \braket{\psi_{\boldsymbol{\theta}'}}{\psi_{\boldsymbol{\theta}}}
    &=
    \sum_{\vb{s}}
    \sqrt{\mathcal{P}_{\boldsymbol{\theta}}(\vb{s})\mathcal{P}_{\boldsymbol{\theta}'}(\vb{s})}
    e^{i\alpha_{\boldsymbol{\theta}';\boldsymbol{\theta}}(\vb{s})}
    =
    \E_{\vb{s}\sim \mathcal{P}_{\boldsymbol{\theta}}}
    \left[
    \sqrt{r_{\boldsymbol{\theta}';\boldsymbol{\theta}}(\vb{s})}
    e^{i\alpha_{\boldsymbol{\theta}';\boldsymbol{\theta}}(\vb{s})}
    \right].
\end{align}
Since $|\alpha_{\boldsymbol{\theta}';\boldsymbol{\theta}}(\vb{s})|\le\delta/2\le\pi/4$,
\begin{align}
    |\braket{\psi_{\boldsymbol{\theta}'}}{\psi_{\boldsymbol{\theta}}}|
    &\ge
    \Re\braket{\psi_{\boldsymbol{\theta}'}}{\psi_{\boldsymbol{\theta}}}\\
    &=
    \E_{\vb{s}\sim \mathcal{P}_{\boldsymbol{\theta}}}
    \left[
    \sqrt{r_{\boldsymbol{\theta}';\boldsymbol{\theta}}(\vb{s})}\cos\alpha_{\boldsymbol{\theta}';\boldsymbol{\theta}}(\vb{s})
    \right]\\
    &\ge
    \cos\pqty{\frac{\delta}{2}}
    \E_{\vb{s}\sim \mathcal{P}_{\boldsymbol{\theta}}}
    \left[
    \sqrt{r_{\boldsymbol{\theta}';\boldsymbol{\theta}}(\vb{s})}
    \right].
    \label{eq:phase_bound_overlap}
\end{align}
It remains to lower bound the amplitude factor. Set
$a_0=1-\epsilon$ and $b_0=1+\epsilon$. Since
$r_{\boldsymbol{\theta}';\boldsymbol{\theta}}(\vb{s})\in[a_0,b_0]$ and $\sqrt{x}$ is concave, it lies above its chord on
$[a_0,b_0]$:
\begin{equation}
    \sqrt{r_{\boldsymbol{\theta}';\boldsymbol{\theta}}(\vb{s})}
    \ge
    \frac{b_0-r_{\boldsymbol{\theta}';\boldsymbol{\theta}}(\vb{s})}{b_0-a_0}\sqrt{a_0}
    +
    \frac{r_{\boldsymbol{\theta}';\boldsymbol{\theta}}(\vb{s})-a_0}{b_0-a_0}\sqrt{b_0}.
\end{equation}
Taking expectations, using linearity of expectation and 
\begin{equation}
    \E_{\vb{s}\sim \mathcal{P}_{\boldsymbol{\theta}}}[r_{\boldsymbol{\theta}';\boldsymbol{\theta}}(\vb{s})]
    =
    \sum_{\vb{s}}\mathcal{P}_{\boldsymbol{\theta}'}(\vb{s})
    =
    1,
\end{equation}
gives
\begin{equation}
    \E_{\vb{s}\sim \mathcal{P}_{\boldsymbol{\theta}}}\bqty{\sqrt{r_{\boldsymbol{\theta}';\boldsymbol{\theta}}(\vb{s})}}
    \ge
    \frac{b_0-1}{b_0-a_0}\sqrt{a_0}
    +
    \frac{1-a_0}{b_0-a_0}\sqrt{b_0}
    =
    \frac{
    \sqrt{1-\epsilon}
    +
    \sqrt{1+\epsilon}
    }{2}.
    \label{eq:amplitude_bound_bhattacharyya}
\end{equation}
Combining Eqs.~\eqref{eq:phase_bound_overlap} and
\eqref{eq:amplitude_bound_bhattacharyya}, then squaring, gives
Eq.~\eqref{eq:infidelity_bound_epsilon_delta}. The compact form follows from
\[
    \left(
    \frac{\sqrt{1-\epsilon}+\sqrt{1+\epsilon}}{2}
    \right)^2
    =
    \frac{1+\sqrt{1-\epsilon^2}}{2}.
\]
The small-trust-region expansion follows from
$\cos^2\delta=1-\delta^2+O(\delta^4)$ and
$(1+\sqrt{1-\epsilon^2})/2=1-\epsilon^2/4+O(\epsilon^4)$.
\end{proof}

\subsection{PWO Surrogate Upper Bound and First-Order Consistency}
\label{app:pwo_surrogate_bounds}

We now relate the exact complex-amplitude decomposition to the practical PWO
surrogate. The exact decomposition contains the coefficients
$2\sqrt r\cos\alpha$ and $2\sqrt r\sin\alpha$, whereas PWO uses the simpler
first-order surrogates $r$ and $r \phi=2r\alpha$. The following lemma controls
the phase approximation error.
\\
\begin{lemma}[Quadratic error of the PWO phase coefficient]
\label{lem:phase-coefficient}
For every $t\ge0$ and $\alpha\in[-\pi,\pi]$,
\begin{equation}
    \left|
    2t\sin\alpha-2t^2\alpha
    \right|
    \le
    4\pi^2
    \left|
    te^{i\alpha}-1
    \right|^2.
    \label{eq:phase-coefficient-bound}
\end{equation}
\end{lemma}

\begin{proof}
Set
\begin{equation}
    d(t,\alpha)
    :=
    |te^{i\alpha}-1|^2
    =
    (t-1)^2+2t(1-\cos\alpha).
\end{equation}
We consider three ranges of $t$. If $0\le t\le1/2$, then
$d(t,\alpha)\ge(1-t)^2\ge1/4$, while
\begin{equation}
    |2t\sin\alpha-2t^2\alpha|
    \le
    2t+2\pi t^2
    \le
    1+\frac{\pi}{2}
    \le
    4\pi^2 d(t,\alpha).
\end{equation}
If $t\ge2$, then $d(t,\alpha)\ge(t-1)^2\ge t^2/4$, and
\begin{equation}
    |2t\sin\alpha-2t^2\alpha|
    \le
    2t+2\pi t^2
    \le
    4\pi^2 d(t,\alpha).
\end{equation}
It remains to consider $1/2\le t\le2$. We write
\begin{align}
    2t\sin\alpha-2t^2\alpha
    =
    2t(\sin\alpha-\alpha)
    +
    2t\alpha(1-t),
\end{align}
and use the bounds, derived from the Taylor expansions,
$|\sin\alpha-\alpha|\le \alpha^2/2$ and
$1-\cos\alpha\ge2\alpha^2/\pi^2$ on $[-\pi,\pi]$. Thus
\begin{equation}
    d(t,\alpha)
    \ge
    (t-1)^2+\frac{4t\alpha^2}{\pi^2}.  
\end{equation}
The term $t\alpha^2$ is bounded by $(\pi^2/4)d(t,\alpha)$. For the mixed term,
let
$x=|1-t|$ and $y=2\sqrt t|\alpha|/\pi$. Then
$d(t,\alpha)\ge x^2+y^2$ and
\begin{equation}
    2t|\alpha||1-t|
    =
    \pi\sqrt t\,xy
    \le
    \frac{\pi\sqrt t}{2}(x^2+y^2)
    \le
    \frac{\pi\sqrt2}{2}d(t,\alpha).
\end{equation}
Combining these bounds gives
\begin{equation}
    |2t\sin\alpha-2t^2\alpha|
    \le
    \left(
    \frac{\pi^2}{4}
    +
    \frac{\pi\sqrt2}{2}
    \right)d(t,\alpha)
    \le
    4\pi^2d(t,\alpha).
\end{equation}
\end{proof}

Define the unclipped PWO surrogates
\begin{align}
    S_{\mathrm{mod}}(\boldsymbol{\theta}';\boldsymbol{\theta})
    &:=
    \E_{\vb{s}\sim\mathcal P_{\boldsymbol{\theta}}}
    \left[
    r_{\boldsymbol{\theta}';\boldsymbol{\theta}}(\vb{s})
    A_{\boldsymbol{\theta}}^{\mathrm{R}}(\vb{s})
    \right],
    \label{eq:unclipped-modulus-surrogate}
    \\
    S_{\mathrm{arg}}(\boldsymbol{\theta}';\boldsymbol{\theta})
    &:=
    \E_{\vb{s}\sim\mathcal P_{\boldsymbol{\theta}}}
    \left[
    r_{\boldsymbol{\theta}';\boldsymbol{\theta}}(\vb{s})
     \phi_{\boldsymbol{\theta}';\boldsymbol{\theta}}(\vb{s})
    A_{\boldsymbol{\theta}}^{\mathrm{I}}(\vb{s})
    \right].
    \label{eq:unclipped-phase-surrogate}
\end{align}

\begin{theorem}[PWO modulus--phase surrogate upper bound]
\label{thm:pwo-surrogate-bound}
Assume $A_{\boldsymbol{\theta}}^{\mathrm{R}}$ and $A_{\boldsymbol{\theta}}^{\mathrm{I}}$ are bounded. Then
\begin{equation}
    E[\psi_{\boldsymbol{\theta}'}]
    -E[\psi_{\boldsymbol{\theta}}]
    \le
    S_{\mathrm{mod}}(\boldsymbol{\theta}';\boldsymbol{\theta})
    +
    S_{\mathrm{arg}}(\boldsymbol{\theta}';\boldsymbol{\theta})+
    2C_{\boldsymbol{\theta}}
    \left(
    1-
    \sqrt{
    1-\mathcal I(
    \psi_{\boldsymbol{\theta}},
    \psi_{\boldsymbol{\theta}'}
    )
    }
    \right),
    \label{eq:pwo-surrogate-upper-bound}
\end{equation}
where one may take
\begin{equation}
    C_{\boldsymbol{\theta}}
    =
    \norm{\hat H-E[\psi_{\boldsymbol{\theta}}]\I}_{\infty}
    +
    \norm{A_{\boldsymbol{\theta}}^{\mathrm{R}}}_{\infty}
    +
    4\pi^2\norm{A_{\boldsymbol{\theta}}^{\mathrm{I}}}_{\infty}.
    \label{eq:pwo-bound-constant}
\end{equation}
In particular,
\begin{align}
    E[\psi_{\boldsymbol{\theta}'}]
    -
    E[\psi_{\boldsymbol{\theta}}]
    &\le
    S_{\mathrm{mod}}(\boldsymbol{\theta}';\boldsymbol{\theta})
    +
    S_{\mathrm{arg}}(\boldsymbol{\theta}';\boldsymbol{\theta})
    +
    2C_{\boldsymbol{\theta}}
    \mathcal I(
    \psi_{\boldsymbol{\theta}},
    \psi_{\boldsymbol{\theta}'}
    ).
    \label{eq:pwo-linear-infidelity-bound}
\end{align}
\end{theorem}

\begin{proof}
For readability, write
$r=r_{\boldsymbol{\theta}';\boldsymbol{\theta}}$,
$\alpha=\alpha_{\boldsymbol{\theta}';\boldsymbol{\theta}}$,
$ \phi=2\alpha$, and $z=\sqrt r e^{i\alpha}$. Since
\[
    |z-1|^2
    =
    1+r-2\sqrt r\cos\alpha,
\]
and
$\E_{\vb{s}\sim\mathcal P_{\boldsymbol{\theta}}}[A_{\boldsymbol{\theta}}^{\mathrm{R}}]=0$, we have
\begin{align}
    2\E_{\mathcal P_{\boldsymbol{\theta}}}
    [
    \sqrt r\cos\alpha\,A_{\boldsymbol{\theta}}^{\mathrm{R}}
    ]
    &=
    \E_{\mathcal P_{\boldsymbol{\theta}}}
    [
    rA_{\boldsymbol{\theta}}^{\mathrm{R}}
    ]
    -
    \E_{\mathcal P_{\boldsymbol{\theta}}}
    [
    |z-1|^2A_{\boldsymbol{\theta}}^{\mathrm{R}}
    ]  \\
    &\le
    S_{\mathrm{mod}}
    +
    \norm{A_{\boldsymbol{\theta}}^{\mathrm{R}}}_\infty
    \E_{\mathcal P_{\boldsymbol{\theta}}}
    [
    |z-1|^2
    ].
    \label{eq:modulus_surrogate_error}
\end{align}
For the phase term, apply Lemma~\ref{lem:phase-coefficient} with
$t=\sqrt r$:
\begin{align}
    2\E_{\mathcal P_{\boldsymbol{\theta}}}
    [
    \sqrt r\sin\alpha\,A_{\boldsymbol{\theta}}^{\mathrm{I}}
    ]
    &=
    \E_{\mathcal P_{\boldsymbol{\theta}}}
    [
    r \phi A_{\boldsymbol{\theta}}^{\mathrm{I}}
    ]
    +
    \E_{\mathcal P_{\boldsymbol{\theta}}}
    [
    (2\sqrt r\sin\alpha-2r\alpha)A_{\boldsymbol{\theta}}^{\mathrm{I}}
    ]  \\
    &\le
    S_{\mathrm{arg}}
    +
    4\pi^2
    \norm{A_{\boldsymbol{\theta}}^{\mathrm{I}}}_\infty
    \E_{\mathcal P_{\boldsymbol{\theta}}}
    [
    |z-1|^2
    ].
    \label{eq:phase_surrogate_error}
\end{align}
Finally,
\begin{align}
    \E_{\mathcal P_{\boldsymbol{\theta}}}
    [
    |z-1|^2
    ]
    &=
    \norm{
    \psi_{\boldsymbol{\theta}'}
    -
    \psi_{\boldsymbol{\theta}}
    }^2  \\
    &=
    2\left(
    1-
    \sqrt{
    1-\mathcal I(
    \psi_{\boldsymbol{\theta}},
    \psi_{\boldsymbol{\theta}'}
    )
    }
    \right),
\end{align}
where we used the global-phase convention
Eq.~\eqref{eq:parallel_transport_gauge}. Combining
Eqs.~\eqref{eq:exact-complex-upper-bound},
\eqref{eq:modulus_surrogate_error}, and
\eqref{eq:phase_surrogate_error} proves
Eq.~\eqref{eq:pwo-surrogate-upper-bound}. The linear-infidelity form follows
from $1-\sqrt{1-x}\le x$.
\end{proof}

\begin{theorem}[First-order consistency of the PWO surrogate]
\label{thm:pwo-first-order-consistency-main}
Assume that $\psi_{\boldsymbol{\theta}}$ is differentiable in a neighborhood of
the reference parameters. For the phase surrogate, detach the importance ratio
from the gradient computation:
\[
    S_{\mathrm{arg}}(\boldsymbol{\theta}';\boldsymbol{\theta})
    =
    \E_{\vb{s}\sim\mathcal P_{\boldsymbol{\theta}}}
    \left[
    \operatorname{sg}
    \left(
    r_{\boldsymbol{\theta}';\boldsymbol{\theta}}(\vb{s})
    \right)
     \phi_{\boldsymbol{\theta}';\boldsymbol{\theta}}(\vb{s})
    A_{\boldsymbol{\theta}}^{\mathrm{I}}(\vb{s})
    \right].
\]
Then
\begin{equation}
    \left.
    \grad_{\boldsymbol{\theta}'}
    \left[
    S_{\mathrm{mod}}(\boldsymbol{\theta}';\boldsymbol{\theta})
    +
    S_{\mathrm{arg}}(\boldsymbol{\theta}';\boldsymbol{\theta})
    \right]
    \right|_{\boldsymbol{\theta}'=\boldsymbol{\theta}}
    =
    \grad_{\boldsymbol{\theta}}
    E[\psi_{\boldsymbol{\theta}}].
    \label{eq:pwo-gradient-consistency}
\end{equation}
The same derivative is obtained from the clipped objectives at the reference
point, since $r=1$ and $ \phi=0$ lie in the interior of the clipping
intervals.
\end{theorem}

\begin{proof}
At $\boldsymbol{\theta}'=\boldsymbol{\theta}$,
\begin{align}
    \left.
    \grad_{\boldsymbol{\theta}'}
    r_{\boldsymbol{\theta}';\boldsymbol{\theta}}(\vb{s})
    \right|_{\boldsymbol{\theta}'=\boldsymbol{\theta}}
    &=
    \grad_{\boldsymbol{\theta}}
    \log\mathcal P_{\boldsymbol{\theta}}(\vb{s})
    =
    2\grad_{\boldsymbol{\theta}}
    \log|\psi_{\boldsymbol{\theta}}(\vb{s})|,
    \\
    \left.
    \grad_{\boldsymbol{\theta}'}
     \phi_{\boldsymbol{\theta}';\boldsymbol{\theta}}(\vb{s})
    \right|_{\boldsymbol{\theta}'=\boldsymbol{\theta}}
    &=
    2\grad_{\boldsymbol{\theta}}
    \arg\psi_{\boldsymbol{\theta}}(\vb{s}).
\end{align}
Therefore the gradient of $S_{\mathrm{mod}}$ recovers the
$2A_{\boldsymbol{\theta}}^{\mathrm{R}}\grad\log|\psi_{\boldsymbol{\theta}}|$ term in
Eq.~\eqref{eq:complex-vmc-gradient}, while the detached phase surrogate
recovers the
$2A_{\boldsymbol{\theta}}^{\mathrm{I}}\grad\arg\psi_{\boldsymbol{\theta}}$ term. Thus the sum of the
two surrogate gradients equals the VMC gradient.
\end{proof}

\firstorderconsistency*

\label{app::pwo-first-order-consistency-main-app}
\begin{proof}
    Since clipping is inactive to
first order at $r=1$ and $ \phi=0$, the clipped objectives have the same
derivative at the reference point and the proof of Theorem \ref{thm:pwo-first-order-consistency-main} still applies.
\end{proof}

\subsection{Clipped PWO Improvement Certificate}
\label{app:pwo_clipped_improvement_certificate}

We finally combine the PWO surrogate bound with the amplitude--phase
infidelity control. This gives the deterministic population-level certificate
corresponding to the clipped PWO objective.

Define the clipped losses
\begin{align}
    L_{\mathrm{mod}}^{\mathrm{clip}}
    (\boldsymbol{\theta}';\boldsymbol{\theta})
    &:=
    \E_{\mathcal P_{\boldsymbol{\theta}}}
    \left[
    \max
    \left(
    rA_{\boldsymbol{\theta}}^{\mathrm{R}},
    \operatorname{clip}(r,1-\epsilon,1+\epsilon)A_{\boldsymbol{\theta}}^{\mathrm{R}}
    \right)
    \right],
    \\
    L_{\mathrm{arg}}^{\mathrm{clip}}
    (\boldsymbol{\theta}';\boldsymbol{\theta})
    &:=
    \E_{\mathcal P_{\boldsymbol{\theta}}}
    \left[
    \operatorname{sg}(r)
    \max
    \left(
     \phi A_{\boldsymbol{\theta}}^{\mathrm{I}},
    \operatorname{clip}( \phi,-\delta,\delta)A_{\boldsymbol{\theta}}^{\mathrm{I}}
    \right)
    \right],
\end{align}
where, as above,
$r=r_{\boldsymbol{\theta}';\boldsymbol{\theta}}$ and
$ \phi= \phi_{\boldsymbol{\theta}';\boldsymbol{\theta}}=2\alpha$.
\\
\pwoimprovement*

\begin{proof}[Proof of Corollary~\ref{cor:pwo-clipped-improvement}]
The clipped energy bound first follows from the linear-infidelity bound~\eqref{eq:pwo-linear-infidelity-bound} and the pointwise inequalities
\begin{equation}
L^{\mathrm{clip}}_{\mathrm{mod}}(\boldsymbol{\theta})
\geq
S_{\mathrm{mod}}(\boldsymbol{\theta};\boldsymbol{\theta}_{\mathrm{old}}),
\qquad
L^{\mathrm{clip}}_{\mathrm{arg}}(\boldsymbol{\theta})
\geq
S_{\mathrm{arg}}(\boldsymbol{\theta};\boldsymbol{\theta}_{\mathrm{old}}),
\end{equation}
which hold because each clipped objective takes the maximum of the unclipped term and its clipped counterpart, and $\operatorname{sg}(r)$ has the same forward value as $r$. Thus
\begin{equation}
E[\psi_{\boldsymbol{\theta}}]-E[\psi_{\boldsymbol{\theta}_{\mathrm{old}}}]
\leq
L^{\mathrm{clip}}_{\mathrm{mod}}(\boldsymbol{\theta})
+
L^{\mathrm{clip}}_{\mathrm{arg}}(\boldsymbol{\theta})
+
2C_{\boldsymbol{\theta}_{\mathrm{old}}}
\mathcal I(\psi_{\boldsymbol{\theta}_{\mathrm{old}}},\psi_{\boldsymbol{\theta}}).
\end{equation}
By Lemma~\ref{lem:amplitude_phase_controls_infidelity}, the assumptions $r_{\boldsymbol{\theta}}\in[1-\epsilon,1+\epsilon]$ and $| \phi_{\boldsymbol{\theta}}|\leq\delta$ imply
\begin{equation}
\mathcal I(\psi_{\boldsymbol{\theta}_{\mathrm{old}}},\psi_{\boldsymbol{\theta}})
\leq
1-
\frac{1+\sqrt{1-\epsilon^2}}{2}
\cos^2\!\left(\frac{\delta}{2}\right),
\end{equation}
because the wrapped phase itself is $ \phi_{\boldsymbol{\theta}}/2$. Substituting this estimate into the previous display gives Eq.~\eqref{eq:pwo_main_bound}.
\end{proof}

PWO only encourages, but does not guarantee, the global trust-region conditions required by the certificate. First, the initial step is always taken from the unclipped objective, since clipping is inactive at $(\boldsymbol{\theta}=\boldsymbol{\theta}_{\mathrm{old}})$. Second, even after clipping activates on a sampled configuration, an update driven by other unclipped configurations can still change its amplitude or phase because all configurations share parameters. Thus, clipping controls the empirical surrogate on the sampled batch, but does not impose pointwise bounds globally. The certificate should therefore be read as a conditional guarantee for realized updates that satisfy the amplitude and phase trust regions over the full support.

\newpage

\section{Additional Sections}
\subsection{Numerical Cost of Stochastic Reconfiguration and Improvements}
\label{app:sr_analysis}

A significant bottleneck arises when using the original SR formulation for optimization. The vanilla implementation involves inverting a $P \times P$ matrix, where $P$ denotes the number of parameters. Alternatively, \cite{Chen_2024} showed that this inversion problem can be recast into one that inverts a matrix of size $M \times M$, where $M$ is the number of samples used per iteration in the Monte Carlo estimation -- this goes by the name of minSR and builds upon ideas from the machine learning optimization literature \citep{ren2019efficient}.
The required operations for minSR are $O(M^2P) + O(M^3)$ instead of $O(P^3)$, and the memory usage is only $O(MP)$ instead of $O(P^2)$. It was used to train a Vision Transformer with around $3 \times 10^5$ parameters, giving the lowest energy so far for the difficult benchmark of the Majumdar–Ghosh point of the $10 \times 10$ J1-J2 Heisenberg system~\citep{rende2024simple-f44}. 

Despite the improvements brought by minSR, several implementations of NQS optimization, especially for electronic structure problems, employ KFAC \citep{martens2015optimizing} instead of SR. The $O(M^3)$ cost can still be prohibitive when training networks for problems that require a large number of samples; KFAC provides a practical compromise: it avoids explicit matrix inversion through block-diagonal approximation, offering scalability at the potential cost of missing information about the geometry of the Hilbert space.  KFAC is commonly used for applications to fermionic systems \citep{choo2020fermionic}, where NQS achieved remarkable results, including architectures like FermiNet \citep{pfau2020ab}, SchNet \citep{schutt2018schnet}, PauliNet \citep{hermann2020deep}, and more recent transformer-based approaches \citep{glehn2023a}.

Recent advances have aimed to improve upon both approaches. These include SPRING, which combines minSR and KFAC to avoid explicit inversion \citep{goldshlager2024kaczmarz}, methods that exploit block-diagonal structure in the quantum geometric tensor to improve conditioning and scalability \citep{nys2024ab, shokry2025less}, and variational Lanczos techniques that accelerate convergence by extracting information across multiple eigenstates \citep{wang2026generalized}. For a more detailed analysis of improvements in second-order optimization of NQS, the reader can refer to \cite{drissi2024second}.

\subsection{Relative Speed of PWO per Iteration}
\label{app:pwo_speed}

\begin{figure}[h]
    \centering

    \begin{subfigure}[t]{0.48\textwidth}
        \centering
        \includegraphics[width=\textwidth]{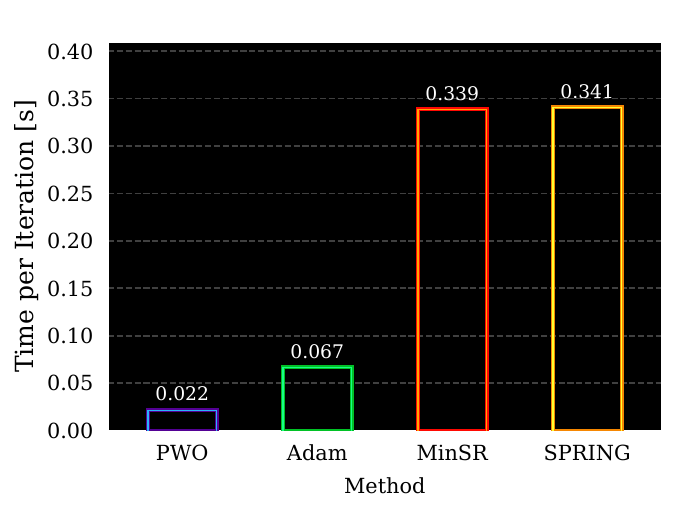}
        \caption{Wall-clock time per optimizer iteration.}
        \label{fig:time_per_iteration_histogram}
    \end{subfigure}
    \hfill
    \begin{subfigure}[t]{0.48\textwidth}
        \centering
        \includegraphics[width=\textwidth]{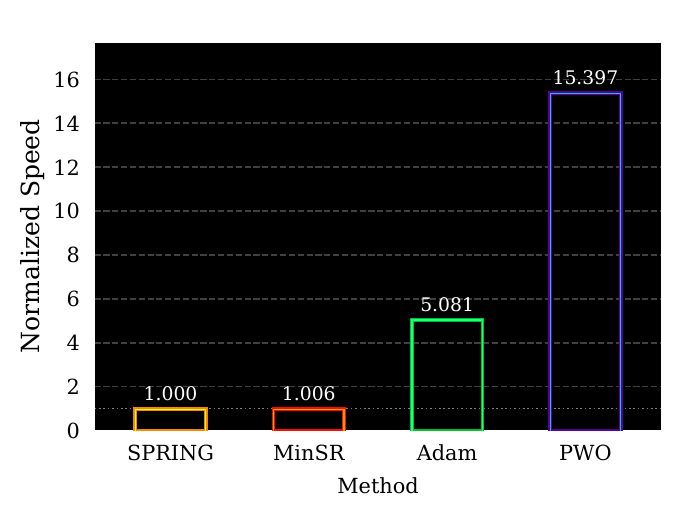}
        \caption{Normalized iteration speed relative to SPRING.}
        \label{fig:normalized_speed_histogram}
    \end{subfigure}

    \caption{
        Per-iteration computational cost and normalized iteration speed for PWO and the baseline optimizers on the Heisenberg $J_1$--$J_2$ chain with the default architecture \ref{app:architecture}. 
        Both measurements are obtained on a single NVIDIA L40S GPU. 
        Normalized speed is reported as the number of optimizer iterations per unit time relative to SPRING, so higher is faster.
    }
    \label{fig:pwo_iteration_speed}
\end{figure}

Beyond convergence in wall-clock time, it is useful to isolate the computational cost of each optimization update. Figure~\ref{fig:normalized_speed_histogram} compares the normalized iteration speed (optimization steps per unit of time) of PWO against the baselines. The speed advantage over minSR is expected: PWO is a first-order method and does not require constructing or inverting the stochastic-reconfiguration matrix. In contrast, minSR replaces the original $P\times P$ SR solve by an $M\times M$ solve, where $M$ is the number of samples, but it still requires expensive Jacobian contractions and a matrix inversion or linear solve. PWO avoids this curvature computation entirely.

The comparison with Adam is more subtle. A single PWO update has additional bookkeeping relative to Adam, including probability ratios, clipping terms, and wrapped phase increments. However, in VMC the dominant cost is often not the optimizer algebra itself, but sampling configurations and evaluating the corresponding local energies. PWO computes local energies and advantages once for a sampled batch and then reuses that batch across multiple proximal inner epochs. Adam, in contrast, performs a single update per sampled batch, so the local-energy overhead is paid again for each parameter update. Thus, PWO can amortize the expensive energy computation across several optimization steps, making its effective per-update cost competitive with Adam despite the more structured surrogate objective.

\subsection{Why RWKV?}
\label{app:why_rwkv}

RWKV is a particularly natural architecture for large-scale autoregressive NQS because it combines the expressivity of sequence models with recurrent inference. Unlike Transformer-based autoregressive models, RWKV does not require a growing key-value cache, so its per-token inference cost and memory footprint remain constant in the sequence length \citep{peng2025rwkv7gooseexpressivedynamic}. This is well aligned with autoregressive VMC, where spin configurations are generated sequentially and model evaluations are repeatedly invoked during Monte Carlo optimization. Recurrent architectures have also been successful in autoregressive NQS more broadly, where they provide exact independent sampling from the Born distribution while avoiding Markov-chain mixing issues \citep{hibat2020recurrent, merali2026parallelscanrecurrentneural}. Finally, recent hyperscale RWKV fine-tuning work provides a practical implementation foundation for adapting billion-parameter recurrent models to nonstandard optimization objectives \citep{sarkar2026evolutionstrategieshyperscale}. For these reasons, RWKV is a useful stress test for whether PWO can scale beyond conventional NQS architectures.

\subsection{Limitations of PWO}
\label{app:limitations}

Although PWO is supported by a first-order consistency result and a clipped improvement bound, the theory remains local and conservative: it relies on common-support and boundedness assumptions and does not imply global convergence for neural-network training. Empirically, our results include one-dimensional chains, a frustrated two-dimensional square lattice, and a large-scale RWKV fine-tuning experiment, but they are still not a comprehensive NQS benchmark across Hamiltonian families. In particular, broader validation on larger two-dimensional systems, fermionic models, and electronic-structure problems remains open. Finally, PWO introduces additional optimization hyperparameters, such as amplitude and phase clipping thresholds and the number of inner epochs, whose optimal values may depend on the Hamiltonian, ansatz, system size, and sample budget.

\newpage
\section{Experimental Details}
\label{app:experimental_details}

\subsection{Hyperparameter search}
\label{app:hyperparameter_search}
For each method, we perform a small hyperparameter search over the learning rate using a grid sweep and select the configuration with the best validation performance. For PWO, we additionally tune the clipping parameters that control the amplitude and phase surrogate objectives. Specifically, after selecting an initial learning rate, we independently grid search each clipping parameter while holding the best values found so far fixed. This sequential procedure provides a simple and computationally tractable way to tune the additional PWO-specific hyperparameters without requiring an exhaustive joint sweep over all combinations.

For Adam-based methods, including PWO, we use a cosine one-cycle learning-rate schedule \citep{smith2018superconvergencefasttrainingneural}. The learning rate is first increased to a peak value during a short warmup phase and is then annealed smoothly to a small final value by cosine decay. The peak learning rate is selected by the grid search above, while the other schedule parameters are fixed across runs. For PWO, which performs multiple optimization epochs per sampled batch, we scale the transition horizon by the number of PPO epochs so that the schedule evolves on a comparable outer-iteration timescale. For minSR and SPRING, a constant learning rate performed better than scheduled variants.

Because the Heisenberg and J1--J2 chains are more challenging optimization problems than the transverse-field Ising model, we use longer learning-rate schedules for these settings (compare Appendices \ref{app:ising_hyperparams}-\ref{app:j1j2_hyperparams}). In particular, we allocate more transition steps to the Heisenberg chain than to Ising, and more transition steps to J1--J2 than to Heisenberg. This gives the optimizer a longer annealing horizon on the harder Hamiltonians, while keeping the schedule structure fixed across problems. To improve readability, we report the method-specific hyper-parameters for each Hamiltonian separately. The Ising configuration is shown first, followed by the Heisenberg and J1--J2 settings in the subsequent subsections.

\subsection{Neural Quantum State Architecture}
\label{app:architecture}

Across experiments, we use the same autoregressive recurrent neural network for PWO and minSR, ensuring that performance differences arise from the optimizer rather than the parametrization. Given a spin configuration $\vb{s}\in\{\pm1\}^N$, we map spins to binary tokens, prepend a beginning-of-sequence token, and embed the resulting sequence with learned token embeddings of dimension $32$ and learned positional embeddings. The embedded inputs are projected to dimension $256$, followed by a $\tanh$ nonlinearity and layer normalization. The backbone consists of GRU layers with residual connections between recurrent blocks. From the shared recurrent representation, we use two separate two-layer MLP heads and GELU activations: an amplitude head, which outputs conditional log-probabilities via a log-softmax, and a phase head, which outputs bounded phases through a $\tanh$ nonlinearity scaled by $\pi$. The model has 1.4M parameters. Unless otherwise stated, all methods and Hamiltonians use the same architecture. The shared hyperparameters are summarized in Table~\ref{tab:architecture_hyperparams}.

\begin{table}[h]
\centering
\small
\begin{tabular}{lc}
\toprule
\textbf{Architecture hyperparameter} & \textbf{Value} \\
\midrule
Token embedding dimension & $32$ \\
Site embeddings & Learned \\
Backbone input dimension & $256$ \\
Backbone nonlinearity & $\tanh$ \\
Backbone normalization & Layer normalization \\
Recurrent backbone & GRU \\
Number of GRU layers & $3$ \\
GRU hidden dimension & $256$ \\
Amplitude head & 2-layer MLP \\
Phase head & 2-layer MLP \\
Head hidden dimension & $256$ \\
Head activation & GELU \\
Phase output & $\pi \tanh(\cdot)$ \\
Phase scale & $\pi$ \\
\bottomrule
\end{tabular}
\vspace{0.5em}
\caption{
Shared NQS architecture used across all Hamiltonians and optimization methods. The same autoregressive GRU-based parameterization is used for Adam, PWO, minSR, and SPRING.
}
\label{tab:architecture_hyperparams}
\end{table}

\subsection{Ising Model}\label{app:ising_hyperparams}

\begin{table}[h]
\centering
\small
\label{tab:ising_hyperparams}
\begin{adjustbox}{max width=\linewidth}
\begin{tabular}{lcccc}
\toprule
\textbf{Hyperparameter} & \textbf{Adam} & \textbf{PWO} & \textbf{minSR} & \textbf{SPRING} \\
\midrule
Optimizer / method & Adam & Adam + PWO & minSR & SPRING \\
Learning rate & $10^{-5}$ & $10^{-5}$ & $10^{-2}$ & $10^{-2}$ \\
Peak learning rate & $10^{-4}$ & $10^{-4}$ & -- & -- \\
Transition steps & $5{,}000$ & $20{,}000$ & -- & -- \\
PPO epochs & -- & $4$ & -- & -- \\
PPO clip $\epsilon$ & -- & $10^{-3}$ & -- & -- \\
Advantage normalization & -- & Yes & -- & -- \\
\midrule
Phase loss & -- & $\Delta\phi$ clip & -- & -- \\
Phase coefficient & -- & $1.0$ & -- & -- \\
Phase clip & -- & $0.3$ & -- & -- \\
Center imaginary advantage & -- & Yes & -- & -- \\
Normalize imaginary advantage & -- & Yes & -- & -- \\
Phase Jacobian baseline & -- & Yes & -- & -- \\
\midrule
SR diagonal shift & -- & -- & $10^{-2}$ & $10^{-2}$ \\
NTK / minSR mode & -- & -- & Yes & Yes \\
On-the-fly SR & -- & -- & Yes & Yes \\
SPRING momentum & -- & -- & -- & 0.8 \\
\bottomrule
\end{tabular}
\end{adjustbox}
\vspace{1em}
\caption{
Hyperparameters used for the Ising experiments. All methods use $N=12$, periodic boundary conditions, $J=1$, $h=1$, complex-valued neural quantum states, $1024$ training samples, exact diagonalization for evaluation, and evaluation every $200$ iterations. Adam and PWO use the cosine one-cycle schedule.
}
\end{table}

\subsection{Heisenberg Chain}\label{app:heisenberg_hyperparams}

\begin{table}[h]
\centering
\small
\begin{adjustbox}{max width=\linewidth}
\begin{tabular}{lcccc}
\toprule
\textbf{Hyperparameter} & \textbf{Adam} & \textbf{PWO} & \textbf{minSR} & \textbf{SPRING} \\
\midrule
Optimizer / method & Adam & Adam + PWO & minSR & SPRING \\
LR parameter / constant LR & $10^{-5}$ & $10^{-5}$ & $10^{-3}$ & $10^{-3}$ \\
Peak learning rate & $3\times 10^{-4}$ & $10^{-4}$ & -- & -- \\
Transition steps & $10{,}000$ & $40{,}000$ & -- & -- \\
PPO epochs & -- & $4$ & -- & -- \\
PPO clip $\epsilon$ & -- & $10^{-3}$ & -- & -- \\
Advantage normalization & -- & Yes & -- & -- \\
\midrule
Phase loss & -- & $\Delta\phi$ clip & -- & -- \\
Phase coefficient & -- & $1.0$ & -- & -- \\
Phase clip & -- & $0.3$ & -- & -- \\
Center imaginary advantage & -- & Yes & -- & -- \\
Normalize imaginary advantage & -- & Yes & -- & -- \\
Phase Jacobian baseline & -- & Yes & -- & -- \\
\midrule
SR diagonal shift & -- & -- & $10^{-2}$ & $10^{-2}$ \\
NTK / minSR mode & -- & -- & Yes & Yes \\
On-the-fly SR & -- & -- & Yes & Yes \\
SPRING momentum & -- & -- & -- & $0.8$ \\
\bottomrule
\end{tabular}
\end{adjustbox}
\vspace{0.5em}
\caption{
Hyperparameters for the Heisenberg-chain experiments. All methods use $N=12$, periodic boundary conditions, coupling $J=0.25$, no sign rule, complex-valued neural quantum states, $1024$ training samples, exact evaluation, and evaluation every $200$ iterations. Adam and PWO use a cosine one-cycle learning-rate schedule; minSR and SPRING use a constant learning rate.
}
\label{tab:heisenberg_hyperparams}
\end{table}

\newpage

\subsection{Heisenberg $J_1$--$J_2$ Chain}\label{app:j1j2_hyperparams}

\begin{table}[h]
\centering
\small
\begin{adjustbox}{max width=\linewidth}
\begin{tabular}{lcccc}
\toprule
\textbf{Hyperparameter} & \textbf{Adam} & \textbf{PWO} & \textbf{minSR} & \textbf{SPRING} \\
\midrule
Optimizer / method & Adam & Adam + PWO & minSR & SPRING \\
LR parameter / constant LR & $10^{-5}$ & $10^{-5}$ & $10^{-3}$ & $10^{-3}$ \\
Peak learning rate & $3\times 10^{-4}$ & $10^{-4}$ & -- & -- \\
Transition steps & ${1}0{,}000$ & $40{,}000$ & -- & -- \\
PPO epochs & -- & $4$ & -- & -- \\
PPO clip $\epsilon$ & -- & $10^{-3}$ & -- & -- \\
Advantage normalization & -- & Yes & -- & -- \\
\midrule
Phase loss & -- & $\Delta\phi$ clip & -- & -- \\
Phase coefficient & -- & $1.0$ & -- & -- \\
Phase clip & -- & $0.3$ & -- & -- \\
Center imaginary advantage & -- & Yes & -- & -- \\
Normalize imaginary advantage & -- & Yes & -- & -- \\
Phase Jacobian baseline & -- & Yes & -- & -- \\
\midrule
SR diagonal shift & -- & -- & $10^{-2}$ & $10^{-2}$ \\
NTK / minSR mode & -- & -- & Yes & Yes \\
On-the-fly SR & -- & -- & Yes & Yes \\
SPRING momentum & -- & -- & -- & $0.8$ \\
\bottomrule
\end{tabular}
\end{adjustbox}
\vspace{0.5em}
\caption{
Hyperparameters for the frustrated Heisenberg $J_1$--$J_2$ experiments. All methods use $N=12$, periodic boundary conditions, couplings $J_1=1$ and $J_2=0.5$, no sign rule, complex-valued neural quantum states, $1024$ training samples, exact evaluation, and evaluation every $200$ iterations. Adam and PWO use a cosine one-cycle learning-rate schedule; minSR and SPRING use a constant learning rate.
}
\label{tab:j1j2_hyperparams}
\end{table}

\subsection{Two-dimensional $J_1$--$J_2$ square-lattice experiment}
\label{app:2d-j1j2}

Using the same method hyper-parameters as above, we evaluate PWO on the frustrated spin-$1/2$ $J_1$--$J_2$ Heisenberg model on a two-dimensional square lattice,
\begin{align}
    \hat H_{J_1\text{--}J_2}^{\mathrm{2D}}
    =
    J_1 \sum_{\langle i,j\rangle} \hat{\vb S}_i \cdot \hat{\vb S}_j
    +
    J_2 \sum_{\langle\langle i,j\rangle\rangle} \hat{\vb S}_i \cdot \hat{\vb S}_j ,
\end{align}
where $\langle i,j\rangle$ and $\langle\langle i,j\rangle\rangle$ denote nearest- and next-nearest-neighbor pairs, respectively. We use an $L\times L$ square lattice with $L=10$, periodic boundary conditions, $J_1=1$, and $J_2=0.5$. All runs are restricted to the zero-magnetization sector, $S^z_{\mathrm{tot}}=0$, by masking infeasible autoregressive choices during sampling and evaluation.

For this experiment we use a complex-valued patch-autoregressive transformer designed for two-dimensional lattices. Instead of generating spins one at a time, the model partitions the $10\times 10$ lattice into non-overlapping $2\times 2$ patches. Each patch is represented as a categorical token with vocabulary size $2^4=16$, so the full configuration is generated as a sequence of
\begin{align}
    T = \left(\frac{L}{2}\right)^2 = 25
\end{align}
autoregressive tokens. The wavefunction is factorized over patch tokens as
\begin{align}
    \log \psi_\theta(\vb s)
    =
    \sum_{t=1}^{T}
    \left[
        \frac{1}{2}\log p_\theta(a_t \mid a_{<t})
        +
        i\,\phi_\theta(a_t \mid a_{<t})
    \right],
\end{align}
where the factor $1/2$ corresponds to the Born-rule convention $P_\theta(\vb s)=|\psi_\theta(\vb s)|^2$.

The model prepends a beginning-of-sequence token and embeds patch tokens with learned token embeddings of dimension $64$. We also use learned site embeddings and prefix-count features that encode the partial magnetization constraint. These embeddings are projected to width $96$, normalized with layer normalization, and passed through an $8$-layer causal transformer backbone. Each transformer block uses $6$ attention heads, two-dimensional axial RoPE with base $100$, residual connections, and a feedforward width of $4\times 96=384$.

From the final transformer representation, the model uses separate amplitude and phase heads. Both heads are two-layer MLPs with hidden dimension $192$ and GELU activations. The amplitude head outputs masked conditional log-probabilities through a log-softmax, ensuring exact normalization over feasible patch choices. The phase head outputs centered phase increments using a $\pi\tanh(\cdot)$ parameterization, with small initialization scale $10^{-3}$. This gives an exactly sampleable complex autoregressive neural quantum state adapted to the two-dimensional square lattice.

While the aim of this experiment is to compare optimizers and not to achieve the state-of-the-art variational energy (around -199.0536 \citep{rende2024simple-f44}), we note that the energy of the PWO-trained autoregressive transformer of 1.5M parameters is $-185$ after 30 mins, reaches $-195.6$ after 24 hours, and keeps decreasing. The run is on a single GPU and imposes no symmetries, except for zero-magnetization sampling.

\begin{table}[h]
\centering
\small
\begin{adjustbox}{max width=\linewidth}
\begin{tabular}{lc}
\toprule
\textbf{Hyperparameter} & \textbf{Value} \\
\midrule
Lattice size & $10\times 10$ \\
Number of spins & $100$ \\
Boundary conditions & Periodic \\
Hamiltonian couplings & $J_1=1,\ J_2=0.5$ \\
Magnetization sector & $S^z_{\mathrm{tot}}=0$ \\
\midrule
Patch size & $2\times 2$ \\
Patch vocabulary size & $16$ \\
Autoregressive tokens & $25$ \\
Token embedding dimension & $64$ \\
Transformer width & $96$ \\
Transformer depth & $8$ \\
Attention heads & $6$ \\
Transformer MLP hidden dimension & $384$ \\
RoPE type & 2D axial RoPE \\
RoPE base & $100$ \\
\midrule
Amplitude head & 2-layer MLP \\
Phase head & 2-layer MLP \\
Head hidden dimension & $192$ \\
Phase parameterization & $\pi\tanh(\cdot)$ \\
Phase initialization std. & $10^{-3}$ \\
Prefix-count features & Yes \\
Learned site embeddings & Yes \\
\bottomrule
\end{tabular}
\end{adjustbox}
\vspace{0.5em}
\caption{
Architecture used for the two-dimensional frustrated $J_1$--$J_2$ square-lattice experiment. The model is a complex patch-autoregressive transformer over $2\times 2$ spin patches, with two-dimensional axial RoPE and an exact zero-magnetization constraint enforced through autoregressive masking.
}
\label{tab:2d_j1j2_architecture}
\end{table}

\subsection{Scaling Experiment}
\label{app:scaling_experiment}

To study how performance scales with model capacity, we train three model sizes on the frustrated Heisenberg $J_1$--$J_2$ chain ($N=12$, $J_1=1$, $J_2=0.5$, periodic boundary conditions) using PWO, Adam, and minSR. All three sizes share the same autoregressive GRU architecture described in Appendix~\ref{app:architecture}; they differ only in the number of recurrent layers and hidden dimensions. Table~\ref{tab:scaling_sizes} summarises the sizes and the corresponding total parameter counts.

\begin{table}[h]
\centering
\small
\begin{tabular}{lcccc}
\toprule
\textbf{Size} & \textbf{GRU layers} & \textbf{RNN hidden} & \textbf{Head hidden} & \textbf{Parameters} \\
\midrule
Tiny   & 1 & 64  & 64  & $44{,}452$ \\
Small  & 2 & 128 & 128 & $269{,}156$ \\
Medium & 3 & 256 & 256 & $1{,}456{,}356$ \\
\bottomrule
\end{tabular}
\vspace{0.5em}
\caption{
Model sizes used in the scaling experiment. All models use an embedding dimension of $32$ and are evaluated on the frustrated Heisenberg $J_1$--$J_2$ chain with $N=12$ sites. Parameter counts include all weights and biases of the full autoregressive network (embedding, recurrent backbone, amplitude head, and phase head).
}
\label{tab:scaling_sizes}
\end{table}

The hyperparameters used for each optimization method in the scaling experiment are listed in Table~\ref{tab:scaling_hyperparams}. The schedule and clipping parameters are identical to those of the main $J_1$--$J_2$ experiments (see Tables~\ref{tab:j1j2_hyperparams} and~\ref{tab:architecture_hyperparams}); only the number of transition steps differs to account for the change in model size.

\begin{table}[h]
\centering
\small
\begin{adjustbox}{max width=\linewidth}
\begin{tabular}{lccc}
\toprule
\textbf{Hyperparameter} & \textbf{Adam} & \textbf{PWO} & \textbf{minSR} \\
\midrule
Optimizer / method & Adam & Adam + PWO & minSR \\
Learning rate & $10^{-5}$ & $10^{-5}$ & $10^{-3}$ \\
Peak learning rate & $3\times 10^{-4}$ & $10^{-4}$ & -- \\
Transition steps & $40{,}000$ & $40{,}000$ & -- \\
PPO epochs & -- & $4$ & -- \\
PPO clip $\epsilon$ & -- & $10^{-3}$ & -- \\
Advantage normalization & -- & Yes & -- \\
\midrule
Phase loss & -- & $\Delta\phi$ clip & -- \\
Phase coefficient & -- & $1.0$ & -- \\
Phase clip & -- & $0.3$ & -- \\
Center imaginary advantage & -- & Yes & -- \\
Normalize imaginary advantage & -- & Yes & -- \\
Phase Jacobian baseline & -- & Yes & -- \\
\midrule
SR diagonal shift & -- & -- & $10^{-2}$ \\
NTK / minSR mode & -- & -- & Yes \\
On-the-fly SR & -- & -- & Yes \\
\midrule
Training samples & \multicolumn{3}{c}{$2{,}048$} \\
Seeds & \multicolumn{3}{c}{$10$} \\
\bottomrule
\end{tabular}
\end{adjustbox}
\vspace{0.5em}
\caption{
Hyperparameters for the scaling experiment on the frustrated Heisenberg $J_1$--$J_2$ chain ($N=12$, $J_1=1$, $J_2=0.5$, periodic boundary conditions). All methods and sizes use the same optimization hyperparameters; only the model architecture varies across runs (see Table~\ref{tab:scaling_sizes}). Adam and PWO use a cosine one-cycle learning-rate schedule; minSR uses a constant learning rate.
}
\label{tab:scaling_hyperparams}
\end{table}
\subsection{RWKV-7 on Ising Model}

\begin{table}[h]
\centering
\small
\begin{adjustbox}{max width=\linewidth}
\begin{tabular}{lcc}
\toprule
\textbf{Hyperparameter} & \textbf{Adam} & \textbf{PWO} \\
\midrule
Optimizer / method & Adam & Adam + PWO \\
Model & RWKV-7 & RWKV-7 \\
Model size & $1.5$B & $1.5$B \\
Learning rate & $10^{-5}$ & $10^{-5}$ \\
Transition steps & $1{,}200$ & $4{,}800$ \\
Decay rate & $0.5$ & $0.5$ \\
PPO epochs & -- & $4$ \\
PPO clip $\epsilon$ & -- & $10^{-3}$ \\
Advantage normalization & -- & Yes \\
\midrule
Batch size & $150$ & $150$ \\
Machine power & $2$ & $2$ \\
\midrule
Evaluation samples & $4{,}096$ & $4{,}096$ \\
Evaluation batch size & $128$ & $128$ \\
Exact diagonalization & Yes & Yes \\
\bottomrule
\end{tabular}
\end{adjustbox}
\vspace{0.5em}
\caption{
Hyperparameters for the RWKV-7 fine-tuning experiments on the transverse-field Ising model. Both methods use an autoregressive RWKV-7 model with $1.5$B parameters, $N=12$, periodic boundary conditions, $J=1$, $h=1$, exact diagonalization for evaluation, and samples drawn exactly from the autoregressive Born distribution. Adam corresponds to the single-epoch first-order baseline, while PWO performs four proximal inner epochs per sampled batch.
}
\label{tab:rwkv7_hyperparams}
\end{table}
\clearpage
\section{Additional Figures}
\subsection{Individual Seed Plots for All Hamiltonias}
\label{app:individual_seeds_hams}
\begin{figure}[h]
    \centering
    \captionsetup{font=small,skip=3pt}

    \resizebox{0.99\textwidth}{!}{%
    \begin{minipage}{\textwidth}
    \centering

    {\small\textbf{Ising model}}

    \vspace{0.15em}

    \begin{subfigure}[t]{0.48\textwidth}
        \centering
        \includegraphics[width=\linewidth]{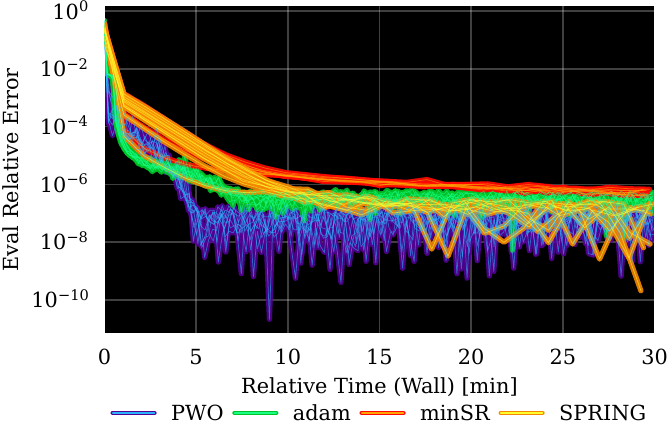}
    \end{subfigure}
    \hfill
    \begin{subfigure}[t]{0.48\textwidth}
        \centering
        \includegraphics[width=\linewidth]{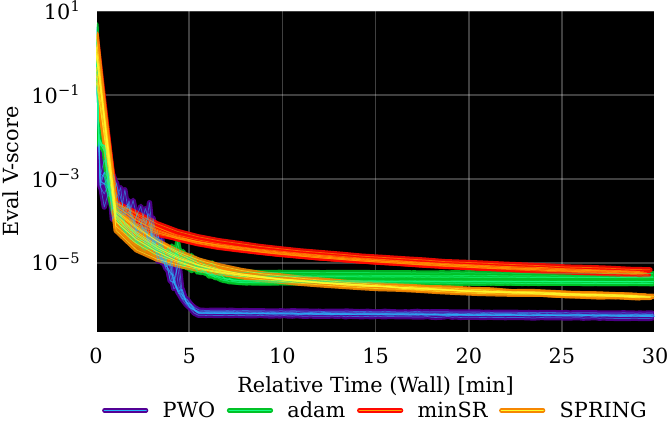}
    \end{subfigure}

    \vspace{0.15em}

    {\small\textbf{Heisenberg chain}}

    \vspace{0.15em}

    \begin{subfigure}[t]{0.48\textwidth}
        \centering
        \includegraphics[width=\linewidth]{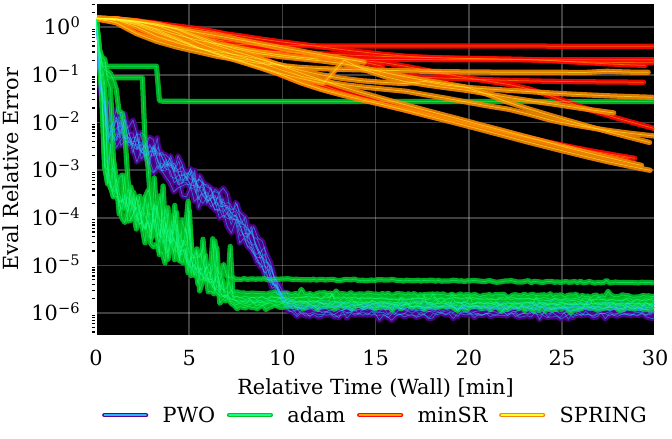}
    \end{subfigure}
    \hfill
    \begin{subfigure}[t]{0.48\textwidth}
        \centering
        \includegraphics[width=\linewidth]{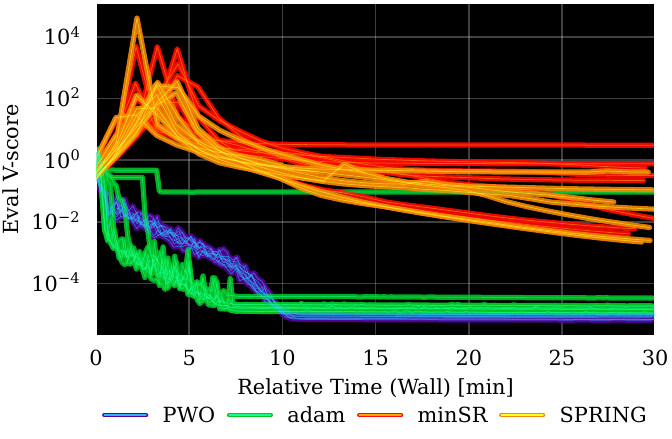}
    \end{subfigure}

    \vspace{0.15em}

    {\small\textbf{$J_1$--$J_2$ Heisenberg model}}

    \vspace{0.15em}

    \begin{subfigure}[t]{0.48\textwidth}
        \centering
        \includegraphics[width=\linewidth]{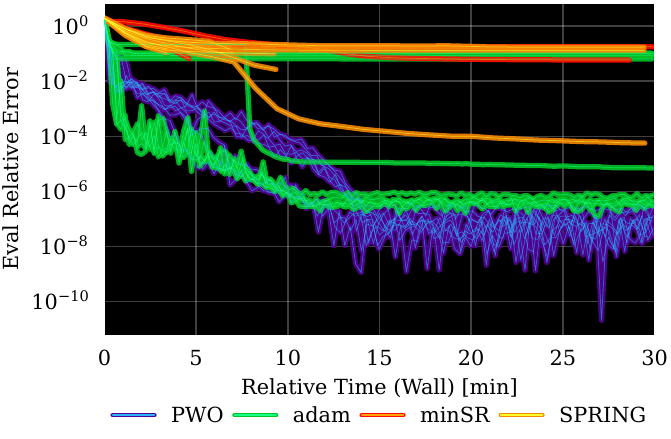}
    \end{subfigure}
    \hfill
    \begin{subfigure}[t]{0.48\textwidth}
        \centering
        \includegraphics[width=\linewidth]{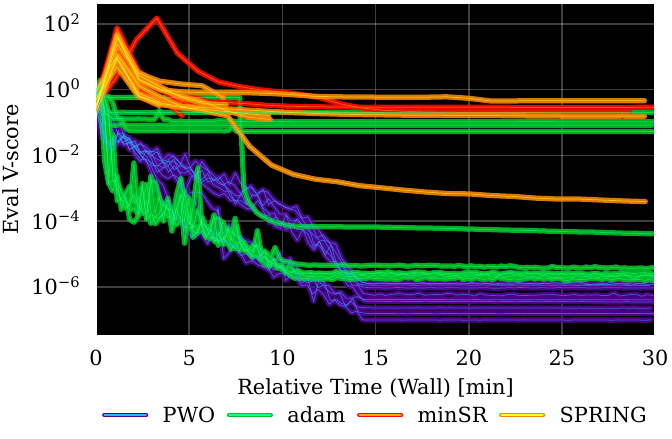}
    \end{subfigure}

    \vspace{0.15em}

    {\small\textbf{Square-lattice $J_1$--$J_2$ Heisenberg model}}

    \vspace{0.15em}

    \begin{subfigure}[t]{0.48\textwidth}
        \centering
        \includegraphics[width=\linewidth]{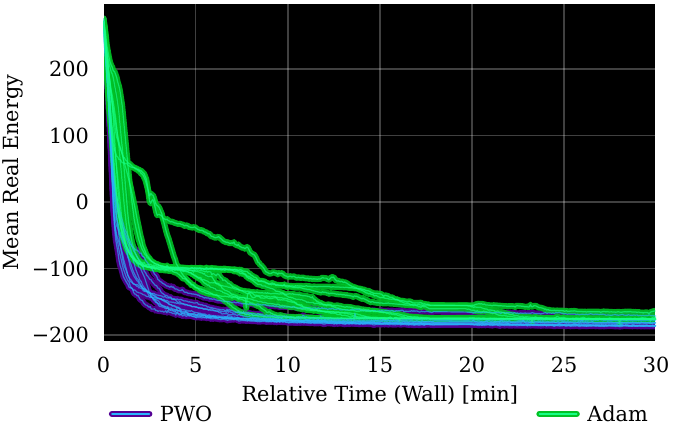}
    \end{subfigure}
    \hfill
    \begin{subfigure}[t]{0.48\textwidth}
        \centering
        \includegraphics[width=\linewidth]{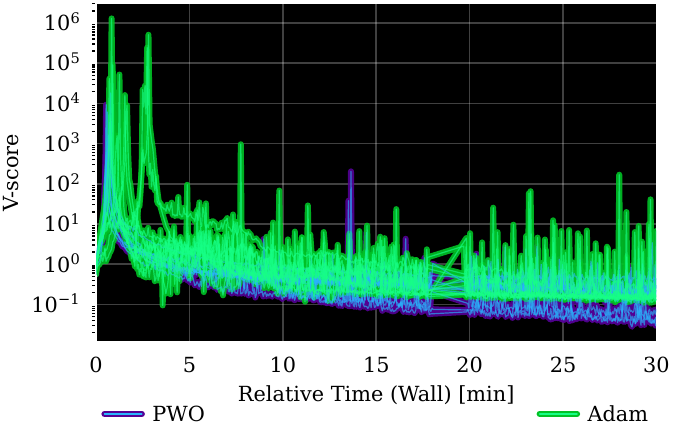}
    \end{subfigure}

    \end{minipage}
    }

    \caption{
    Individual-seed learning curves for all Hamiltonians. Each row corresponds to one Hamiltonian, and the right column reports the V-score.
    }
    \label{fig:appendix_seed_curves}
\end{figure}

\newpage

\subsection{Scaling Samples and System Sizes}
\label{app:additional_scaling}

\begin{figure}[h]
    \centering
    \includegraphics[width=\textwidth]{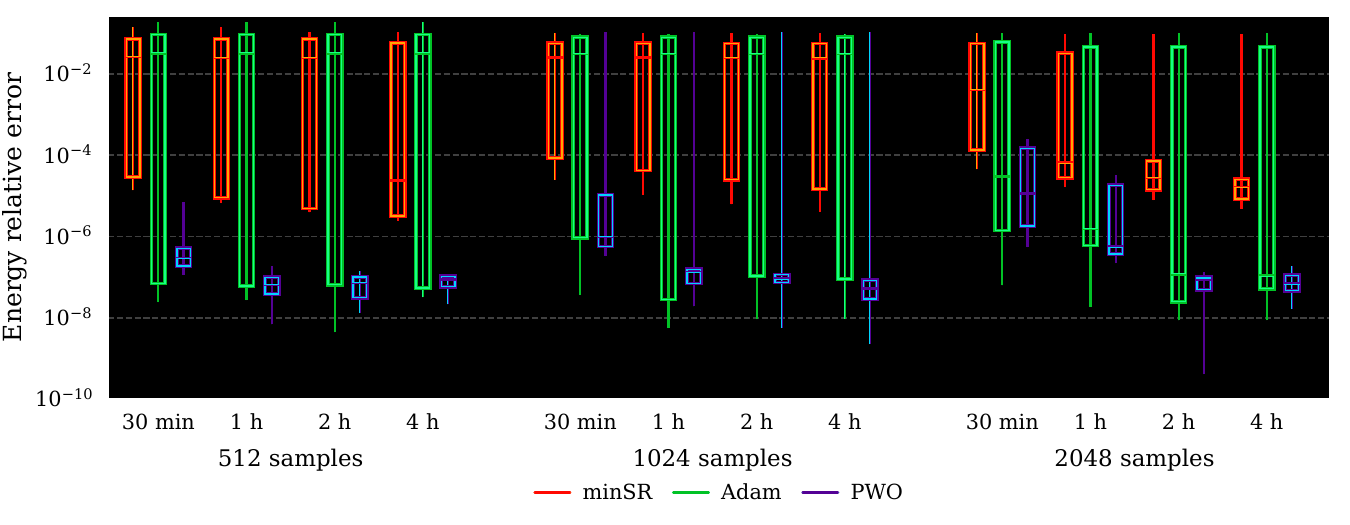}
    \caption{
        Wall-clock scaling comparison across number of samples and optimization methods.
        Boxplots show the interquartile mean relative error over seeds, with boxes indicating the interquartile range and lines indicating the min and max.
        Results are grouped by model size and wall-clock time, and run on a single NVIDIA A100 GPU.
    }
    \label{fig:scaling_plot_samples}
\end{figure}

\begin{figure}[h]
    \centering
    \includegraphics[width=\textwidth]{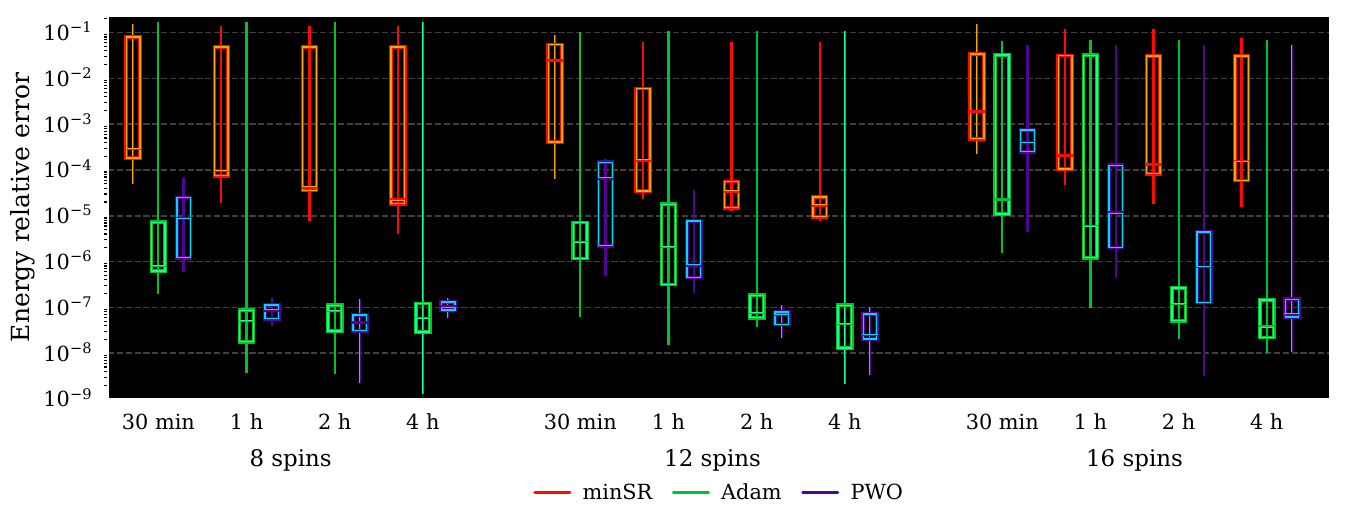}
    \caption{
        Wall-clock scaling comparison across system size and optimization methods.
        Boxplots show the interquartile mean relative error over seeds, with boxes indicating the interquartile range and lines indicating the min and max.
        Results are grouped by model size and wall-clock time, and run on a single NVIDIA A100 GPU.
    }
    \label{fig:scaling_plot_size}
\end{figure}

\newpage

\subsection{Individual Seed Plots for RWKV7 Fine-tuning}\label{app:individual_seeds_LLM}

\begin{figure}[h]
    \centering
    \begin{subfigure}[t]{0.49\linewidth}
        \centering
        \includegraphics[width=\linewidth]{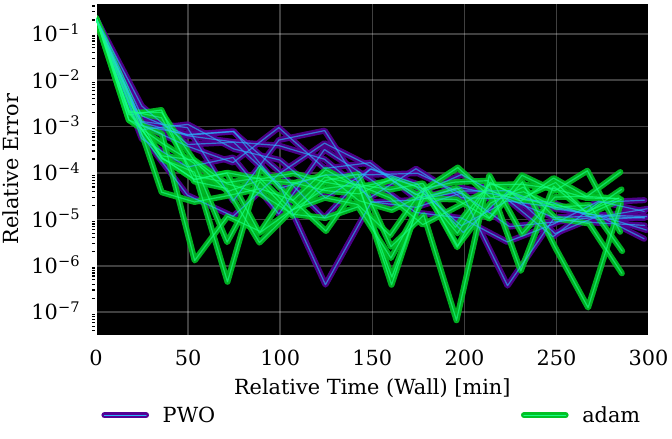}
    \end{subfigure}
    \hfill
    \begin{subfigure}[t]{0.49\linewidth}
        \centering
        \includegraphics[width=\linewidth]{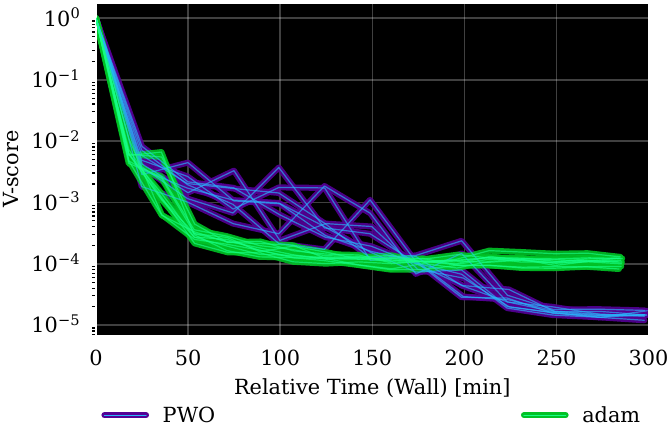}
    \end{subfigure}
    \caption{Individual-seed fine-tuning curves for the $1.5$B-parameter RWKV-7 neural quantum state on the transverse-field Ising model. Each curve corresponds to one random seed, with relative error shown on the left and V-score on the right. PWO consistently remains stable across seeds and reaches lower final error and variance than the Adam baseline, indicating that the proximal objective improves robustness even in the billion-parameter regime.}
\end{figure}

\end{document}